\documentclass{article}
\usepackage{ijcai25}

\usepackage{times}
\usepackage{url}
\usepackage[hidelinks]{hyperref}
\usepackage[utf8]{inputenc}
\usepackage[small]{caption}
\usepackage{tabularx, booktabs}
\usepackage[switch]{lineno}
\usepackage{subfig}
\usepackage{dsfont}
\usepackage{svg}
\usepackage{svg-extract}
\usepackage{siunitx}
\usepackage{amsmath}
\usepackage{enumitem}
\usepackage{gen_settings}

\newcommand{\eps}{\epsilon}

\newcommand{\Prb}{\mathbb{P}}

\newcommand{\state}{\mathbf{s}}
\newcommand{\sz}{{\state_0}}
\newcommand{\s}[1]{{\state_{#1}}}

\newcommand{\st}{{\state_t}}

\newcommand{\sT}{{\state_T}}

\newcommand{\action}{\mathbf{a}}

\newcommand{\at}{{\action_t}}

\newcommand{\traj}[2]{\tau^{#1}_{#2}}

\newcommand{\policy}{\pi}

\newcommand{\policybase}{\pi_{\mathrm{base}}}
\newcommand{\policytask}{\pi_{\mathrm{task}}}
\newcommand{\policyproj}{\pi_{\mathrm{proj}}}

\newif\ifarXiv

\arXivtrue

\ifarXiv
\else
  \linenumbers 
\fi

\title{
SPoRt - Safe Policy Ratio: \\ Certified Training and Deployment of Task Policies in Model-Free RL
}

\author{}

\ifarXiv
 \author{
 Jacques Cloete$^1$
 \and
 Nikolaus Vertovec$^2$\And
 Alessandro Abate$^2$\\
 \affiliations
 $^1$Oxford Robotics Institute, University of Oxford\\
 $^2$Department of Computer Science, University of Oxford\\
 \emails
 jacques@robots.ox.ac.uk,
 \{nikolaus.vertovec, alessandro.abate\}@cs.ox.ac.uk
 }
 \fi

\begin{document}

\maketitle

\begin{abstract}
To apply reinforcement learning to safety-critical applications, we ought to provide safety guarantees during both policy training and deployment. In this work, we present theoretical results that place a bound on the probability of violating a safety property for a new task-specific policy in a model-free, episodic setting. This bound, based on a maximum policy ratio computed with respect to a `safe' base policy, can also be applied to temporally-extended properties (beyond safety) and to robust control problems. To utilize these results, we introduce SPoRt, which provides a data-driven method for computing this bound for the base policy using the scenario approach, and includes Projected PPO, a new projection-based approach for training the task-specific policy while maintaining a user-specified bound on property violation. SPoRt thus enables users to trade off safety guarantees against task-specific performance. Complementing our theoretical results, we present experimental results demonstrating this trade-off and comparing the theoretical bound to posterior bounds derived from empirical violation rates.
\end{abstract}

\section{Introduction}

Reinforcement Learning (RL) is an area of machine learning where an agent is trained to interact with its environment to maximize some (cumulative) reward~\cite{Sutton_Barto_2014,Mason2019}. There has been great interest in applying RL to real-world control problems in fields such as robotics~\cite{Kober2014,Hwangbo2019,Singh_Kumar_Singh_2022}, traffic management~\cite{Chu2020,Vertovec2023,Lee2023} and autonomous driving~\cite{Isele2018,Ma2021,Li2022}, to name just a few. Many of these domains typically fall into the realm of ``safety-critical'' applications, whereby we need to guarantee safety specifications, such as obstacle avoidance. Satisfying safety constraints becomes particularly challenging when we have little to no knowledge of our environment. This problem has been studied in a substantial body of literature, known as model-free safe RL.

Traditional policy gradient algorithms for model-free RL, such as Trust Region Policy Optimization (TRPO)~\cite{Schulman2015} and Proximal Policy Optimization (PPO)~\cite{PPO}, allow the agent to explore any behavior during training, including behaviors that would be considered unsafe; this is unacceptable for safety-critical applications. To encode safety into training, a popular formulation is the Constrained Markov Decision Process (CMDP) \cite{Altman2021}, which includes safety constraints and is typically solved using primal-dual methods ~\cite{Achiam2017} and modifying the trust region to exclude unsafe policy updates ~\cite{Milosevic2024}.
However, the CMDP formulation is limited in its ability to model safety constraints; CMDPs constrain the expected discounted cost return, but for many practical applications we require an explicit bound on the probability that a sampled trajectory violates a safety constraint.

Alternative approaches based on control theory ensure safety by preventing the agent from taking actions that would eventually lead to safety violations; this is achieved using, e.g., Lyapunov and barrier functions \cite{Chow2018}, shielding \cite{Alshiekh2018,Konighofer2023} or safety filters \cite{Hsu2024}. However, these approaches require a model of the environment to predict future safety, and thus are generally limited to model-based setups.
Meanwhile, formal methods-based approaches, such as~\cite{LCRL}, encode safety by leveraging Linear Temporal Logic (LTL)~\cite{Pnueli_1977} as a formal reward-shaping structure. Unlike CMDPs, whereby the original objective is separate from the constraint, LTL formula satisfaction is encoded into the expected return itself, and under certain conditions the trained policy is guaranteed to maximize the probability of LTL formula satisfaction; however, no guarantees can be obtained \emph{during} training.

Since we presume no knowledge of the environment, as in standard model-free RL, we will rely on finite-sample learning to evaluate the agent's ability to remain safe using probably approximately correct (PAC) guarantees. Finite-sample complexity bounds provide the number of samples needed to, with a given confidence, learn some target function with a certain accuracy~\cite{Vidyasagar2003}. Tools from statistical learning theory based on Vapnik Chervonenkis (VC) theory have successfully been able to provide finite sample bounds for learning in unknown environments~\cite{Vidyasagar2003,Tempo2005}, with recent work providing finite sample bounds even under changing target assumptions~\cite{Vertovec2024}. Yet VC-theoretic techniques require the computation of the VC dimension, which is a difficult task for generic optimization problems. Under a convexity assumption, the so-called scenario approach offers a-priori probabilistic feasibility guarantees without resorting to VC theory~\cite{Calafiore2006,Campi2008,Campi2018}. 

The scenario approach traditionally relies on independent and identically distributed (i.i.d.) samples to establish its sample-complexity bounds. This creates a limitation in RL contexts, where the sampling distribution changes as policies are updated. As a result, safety guarantees established for one policy cannot be directly transferred when the policy changes. In this work, we overcome this limitation by extending the PAC guarantees to accommodate policy changes. Specifically, we derive a constraint on how much policies can shift while maintaining safety guarantees, and present SPoRt, an approach for adapting an existing safe policy to improve task-specific performance while maintaining a bound on the probability of safety violation, known prior to deploying or even training the adapted policy; this bound can be tuned by the user to trade off safety and task-specific performance.

Our technical contributions underpinning SPoRt are as follows:
\begin{enumerate}
    \item A data-driven method for obtaining a bound on the probability that a property (e.g. safety), in general expressed as an LTL formula, is violated for trajectories drawn using a given `safe' base policy (Section \ref{data_driven_validation_property_satisfaction}).
    \item Novel theoretical results that provide, for an episodic, model-free RL setup, a prior bound on the probability of property violation for a new task-specific policy, based on a `maximum policy ratio’ computed with respect to the  `safe’ base policy (cf. previous point) (Section \ref{probability_property_violation}).
    \item A projection-based method for constraining the task-specific policy to ensure that this prior bound holds (Section \ref{ensuring_constraint_satisfaction_modified_policy}).
    \item Projected PPO, an algorithm for \emph{training} a new, task-specific policy, while maintaining a user-specified prior bound on property violation, thus trading off safety guarantees for task-specific performance (Section \ref{training_while_maintaining_bound_property_violation}).
\end{enumerate}
We also test SPoRt on a time-bounded reach-avoid property and present experimental results demonstrating the safety-performance trade-off, as well as comparison of the theoretical prior bound to posterior bounds based on empirical violation rates  (Section \ref{case_studies} and \ref{results_discussion}).

All appendices and code\footnote{Link to code: \url{https://github.com/JacquesCloete/sport}.} can be found in the supplementary material, which contains all proofs.

\section{Models, Tasks and Properties}
\label{models}

We consider a model-free episodic RL setup where an agent interacts with an unknown environment modeled as a Markov Decision Process (MDP) \cite{Sutton_Barto_2014}, specified by the tuple $\langle\mathcal{S}, \mathcal{A}, p, \mu, r_\mathrm{task}\rangle$, with a continuous state space $\mathcal{S}$ and continuous action space $\mathcal{A}$. $p(s'|a,s): \mathcal{S} \times \mathcal{A} \to \Delta(\mathcal{S})$ and $\mu(s) \in \Delta(\mathcal{S})$ are the (unknown) state-transition and initial state distributions, respectively. We will consider learning a stochastic policy $\policy(a|s): \mathcal{S} \to \Delta(\mathcal{A})$ in a model-free setup.
We use $\traj{p,\policy}{\st,T} = (\st, \s{t+1}, \ldots, \s{t+T})^{p,\policy}$ to denote a realization of a trajectory of the closed-loop system with state transition distribution $p$, starting at state $\st$ and evolving for $T$ time steps, using policy $\policy$. $r_\mathrm{task}(s,a): \mathcal{S} \times \mathcal{A} \to \mathbb{R}$ is the task-specific reward, which encourages higher task-specific performance (for example, max speed or min time). 

\subsection{Safety as a Temporally-Extended Property}

We define safety in terms of satisfaction of a general temporal property $\varphi$. We denote that a trajectory $\traj{}{}$ satisfies property $\varphi$ (and is therefore safe) by $\traj{}{} \models \varphi$, while $\traj{}{} \not\models \varphi$ indicates that $\traj{}{}$ violates $\varphi$ (and is therefore unsafe). SPoRt addresses problems where the objective is to ensure that $\traj{p,\policy}{\sz \sim \mu,T} \models \varphi$ with high probability, while maximizing the task-specific reward $r_\mathrm{task}$. To evaluate the satisfaction of $\varphi$ we introduce a robustness metric $\varrho^{\varphi}$, which encodes property violation as a real-valued signal that is non-negative only when $\tau \models \varphi$.
\begin{definition}
    \label{def:robustness}
    A \textit{robustness metric} $\varrho^{\varphi}$ is a function $\varrho^{\varphi}(\tau): \mathcal{S}^{n} \to [-a,b],~n \in \mathbb{Z}_{+},~a,b \in \mathbb{R}_{+}$ such that $\varrho^{\varphi}(\tau) \geq 0$ only for trajectories $\tau \in \mathcal{S}^{n}$ that satisfy property $\varphi$ (i.e. $\traj{}{} \models \varphi$).
\end{definition}
Any safety property $\varphi$ can be expressed as a Linear Temporal Logic (LTL) formula \cite{Pnueli_1977}, which ensures the existence of such a metric (see Appendix A.1). Notably, SPoRt extends beyond safety properties to encompass any property $\varphi$ expressible as an LTL formula - the case study deals with `reach-avoid' as we shall see. Accordingly, our theoretical results generalize to 
product MDPs in RL problems under general LTL specifications \cite{LCRL}. While Appendix A.2 provides detailed discussions on these extensions to general LTL formulae and hybrid-state models, for the remainder of the paper (and with no loss in generality) we focus exclusively on safety properties $\varphi$ within continuous-state MDPs, as defined in Section \ref{models}. 

\section{Data-Driven Property Satisfaction}
\label{data_driven_validation_property_satisfaction}




SPoRt provides a method for adapting an existing safe policy ($\policybase$) so as to maximize some task-specific reward ($r_\mathrm{task}$), without violating a given property $\varphi$. 

As a first step, let us evaluate the property satisfaction of given traces for a general policy $\policy$. Given an initial state distribution, state transition distribution and stochastic policy ($\mu,p,\policy$), the value of the robustness metric for an associated trajectory, i.e., $\varrho^{\varphi}(\traj{p,\policy}{\sz \sim \mu,T})$ will be a random variable drawn from some distribution $\Delta_{\mu}^{p,\policy}$ and the probability of satisfying the property $\varphi$ will be encoded by
\begin{equation}
    \Prb\{\varrho^{\varphi}(\traj{p,\policy}{\sz \sim \mu,T}) \in \Delta_{\mu}^{p,\policy}: \varrho^{\varphi}(\traj{p,\policy}{\sz \sim \mu,T}) \geq 0\}.
\end{equation}
SPoRt first bounds the probability of property violation under an existing `safe' base policy $\policybase$, i.e, $\Prb\{\varrho^{\varphi}(\traj{p,\policybase}{\sz \sim \mu,T}) \in \Delta_{\mu}^{p,\policybase}: \varrho^{\varphi}(\traj{p,\policybase}{\sz \sim \mu,T}) < 0\} \leq \eps_{\mathrm{base}}$
using the scenario approach \cite{Campi2018}; we roll out $N$ scenario trajectories $(\traj{p,\policybase}{\sz \sim \mu,T})_{i}$ using $\policybase$ and record them in buffer $\mathcal{D}_{\mu}^{p,\policybase} = \{(\traj{p,\policybase}{\sz \sim \mu,T})_{i}\}_{i=1}^{N}$. We use Theorem \ref{thm:scenario} to obtain an upper bound  $\eps_{\mathrm{base}}$ on the probability of violating $\varphi$:
\begin{theorem}
\label{thm:scenario}
    If $\varrho^{\varphi}((\traj{p,\policy}{\sz \sim \mu,T})_{i}) \geq 0$ for all $N$ scenarios $(\traj{p,\policy}{\sz \sim \mu,T})_{i}$ in $\mathcal{D}_{\mu}^{p,\policy}$ = $\{(\traj{p,\policy}{\sz \sim \mu,T})_{i}\}_{i=1}^{N}$, then with confidence $1-\beta$, where $\beta = \left(1 - \eps\right)^{N}$, the probability of drawing a new scenario $\traj{p,\policy}{\sz \sim \mu,T}$ such that $\varrho^{\varphi}(\traj{p,\policy}{\sz \sim \mu,T}) < 0$ is at most $\eps$.
\end{theorem}
If not all scenarios in $\mathcal{D}_{\mu}^{p,\policybase}$ satisfy $\varphi$, we can leverage results from \cite{Campi_2010} to identify a suitable $\eps_{\mathrm{base}}$ `under k-constraint removal', as follows:  
\begin{corollary}
\label{cor:constraint_removal}
    Assume $k$ scenarios $(\traj{p,\policy}{\sz \sim \mu,T})_{i}$ in buffer $\mathcal{D}_{\mu}^{p,\policy}$ = $\{(\traj{p,\policy}{\sz \sim \mu,T})_{i}\}_{i=1}^{N}$ are such that $\varrho^{\varphi}((\traj{p,\policy}{\sz \sim \mu,T})_{i}) < 0$, then with confidence $ 1 - \beta$, where $\beta = \sum_{i=0}^{k} \binom{N}{i} \epsilon_k^i (1-\epsilon_k)^{N-i}$, the probability of drawing a new scenario $\traj{p,\policy}{\sz \sim \mu,T}$ such that $\varrho^{\varphi}(\traj{p,\policy}{\sz \sim \mu,T}) < 0$ is at most $\epsilon_k$.
\end{corollary}
In both cases, we first collect $N$ scenarios, then choose our confidence $ 1 - \beta$, and then compute the bound $\eps_{\mathrm{base}}$.

\section{
Property Violation under Modified MDPs}
\label{probability_property_violation}


Once $\eps_{\mathrm{base}}$ is obtained, SPoRt safely trains a task-specific policy $\policytask$ so as to maximize the (cumulative) reward $r_{\mathrm{task}}$. For SPoRt to ensure safe training of $\policytask$, we must upper bound the probability of property violation under $\policytask$, i.e., $\Prb\{\varrho^{\varphi}(\traj{p,\policytask}{\sz \sim \mu,T}) \in \Delta_{\mu}^{p,\policytask}: \varrho^{\varphi}(\traj{p,\policytask}{\sz \sim \mu,T}) < 0\}$, by the probability of property violation under $\policybase$, i.e., $\Prb\{\varrho^{\varphi}(\traj{p,\policybase}{\sz \sim \mu,T}) \in \Delta_{\mu}^{p,\policybase}: \varrho^{\varphi}(\traj{p,\policybase}{\sz \sim \mu,T}) < 0\}$, which is upper bounded by $\eps_{\mathrm{base}}$. To do so, we first construct this bound for general $(\mu_{1},p_{1},\policy_{1})$ and $(\mu_{2},p_{2},\policy_{2})$, and then set $(\mu_{1},p_{1},\policy_{1}) = (\mu,p,\policybase)$ and $(\mu_{2},p_{2},\policy_{2}) = (\mu,p,\policytask)$.


Let $\mathcal{S}_{0}, \dots, \mathcal{S}_{T} \subseteq \mathcal{S}$ be a sequence of arbitrary subsets of the state space for each time step in the episode. We begin by characterizing the probability that a sampled trajectory $\traj{p,\policy}{\sz \sim \mu,T} = (\sz \sim \mu, \s{1}, \ldots, \sT)^{p,\policy}$ is such that $\sz \in \mathcal{S}_{0}, \dots, \sT \in \mathcal{S}_{T}$ as a forward recursion, based on work in~\cite{additive_difference,additive_difference_2}. Let $\mathds{1}_{\mathcal{S}_{t}}(s):\mathcal{S} \to \{0, 1\}$ be the indicator function for $s \in \mathcal{S}_{t}$, and define functions $W_{t}^{\mu,p,\policy}(s): \mathcal{S} \to \mathbb{R}_{+}$, characterized as
\begin{gather}
    W_{t+1}^{\mu,p,\policy}(s') = \mathds{1}_{\mathcal{S}_{t+1}}(s') \int_{\mathcal{S}} P^{p,\policy}(s'| s) W_{t}^{\mu,p,\policy}(s)ds\\
    \text{and}
    \quad
    W_{0}^{\mu,p,\policy}(s') = \mathds{1}_{\mathcal{S}_{0}}(s')\mu(s'), \\
    \text{where} \quad P^{p,\policy}(s' | s) = \int_{\mathcal{A}} p(s' | a,s) \policy(a | s) da.
\end{gather}
It holds that $\Prb\{\traj{p,\policy}{\sz \sim \mu,T}: \sz \in \mathcal{S}_{0}, \dots, \sT \in \mathcal{S}_{T}\} = \int_{\mathcal{S}} W_{T}^{\mu,p,\policy}(s) ds$. We then use Theorem \ref{thm:bound_traj_prob} to bound the probability that trajectory $\traj{p_{2},\policy_{2}}{\sz \sim \mu_{2},T}$ remains within $\mathcal{S}_{0}, \dots, \mathcal{S}_{T}$ throughout the episode in terms of the probability that trajectory $\traj{p_{1},\policy_{1}}{\sz \sim \mu_{1},T}$ remains within the same $\mathcal{S}_{0}, \dots, \mathcal{S}_{T}$:
\begin{theorem}
\label{thm:bound_traj_prob}
    Suppose that, for a set of coefficients $\alpha_{t} \in \mathbb{R}_{+}$, we could constrain $\mu_{2}$, $p_{2}$ and $\policy_{2}$ so as to enforce the following bounds for all $t = 1, \dots, T$:
    \begin{align}
        & \int_{\mathcal{S}} P^{p_{2},\policy_{2}}(s' | s) W_{t-1}^{\mu_{1},p_{1},\policy_{1}}(s) ds \\
        \leq & \alpha_{t} \int_{\mathcal{S}} P^{p_{1},\policy_{1}}(s' | s) W_{t-1}^{\mu_{1},p_{1},\policy_{1}}(s) ds 
        \label{bound}
        \\
        \text{and} \quad & \mu_{2}(s') \leq \alpha_{0} \mu_{1}(s'),
        \quad \forall s' \in \mathcal{S}. 
    \end{align}
    It thus holds that
    \begin{align}
        & \Prb\{\traj{p_{2},\policy_{2}}{\sz \sim \mu_{2},T}: \sz \in \mathcal{S}_{0}, \dots, \sT \in \mathcal{S}_{T}\} \\
        \leq & \Prb\{\traj{p_{1},\policy_{1}}{\sz \sim \mu_{1},T}: \sz \in \mathcal{S}_{0}, \dots, \sT \in \mathcal{S}_{T}\} \prod_{t=0}^{T}\alpha_{t}.
    \end{align}
\end{theorem}
We want to make this bound as tight as possible, which is done by minimizing $\alpha_{t}$ subject to Equation \eqref{bound} for all $s' \in \mathcal{S}$ and $t = 1, \dots, T$. Solving this problem is non-trivial due to $p_{1}$ and $p_{2}$ being unknown in a model-free setup. To this end, we introduce Theorem \ref{thm:policy_ratio} to obtain a feasible solution by constraining $p_{2}$ and $\policy_{2}$ in terms of $p_{1}$ and $\policy_{1}$:

\begin{theorem}
\label{thm:policy_ratio}
    Suppose the following constraint holds:
    \begin{gather}
        p_{2}(s'|a,s) \policy_{2}(a|s)
        \leq \alpha_{t} p_{1}(s'|a,s) \policy_{1}(a|s) 
        \; \forall a \in \mathcal{A}, s \in \mathcal{S}.
    \end{gather}
    Thus Equation \eqref{bound} holds for all $s' \in \mathcal{S}$.
\end{theorem}

Assuming stationarity, under this constraint the bound is minimized when $\alpha_{t} = \alpha$ for all $t = 1, \dots, T$. Note also that under this constraint, we find that $\alpha \geq 1$.

It is important to observe that, while correct, this bound can be very conservative for applications with large episode length $T$; a discussion on this conservativeness can be found in Appendix C.1. Alternative bounds from literature suffer from similar blowup \cite{bcSA12,additive_difference_2}. 

Using the results from Theorem \ref{thm:bound_traj_prob} and \ref{thm:policy_ratio}, we can now derive Theorem \ref{thm:task_fail_prob} to obtain a bound on the probability of property violation for ($\mu_{2},p_{2},\policy_{2}$) in terms of ($\mu_{1},p_{1},\policy_{1}$):

\begin{theorem}
\label{thm:task_fail_prob}
    Suppose that
    \begin{equation}
        \Prb\{\varrho^{\varphi}(\traj{p_{1},\policy_{1}}{\sz \sim \mu_{1},T}) \in \Delta_{\mu_{1}}^{p_{1},\policy_{1}}: \varrho^{\varphi}(\traj{p_{1},\policy_{1}}{\sz \sim \mu_{1},T}) < 0\} \leq \eps_{1}
    \end{equation}
    and for all $t = 1, \dots, T$,
    \begin{gather}
        p_{2}(s'|a,s) \policy_{2}(a|s)
        \leq \alpha_{t} p_{1}(s'|a,s) \policy_{1}(a|s) \\
        \text{and} \quad \mu_{2}(s') \leq \alpha_{0} \mu_{1}(s'), \quad \forall a \in \mathcal{A}, s \in \mathcal{S}.
    \end{gather}
    It thus holds that
    \begin{equation}
        \Prb\{\varrho^{\varphi}(\traj{p_{2},\policy_{2}}{\sz \sim \mu_{2},T}) \in \Delta_{\mu_{2}}^{p_{2},\policy_{2}}: \varrho^{\varphi}(\traj{p_{2},\policy_{2}}{\sz \sim \mu_{2},T}) < 0\} \leq \eps_{1} \prod_{t=0}^{T}\alpha_{t}.
    \end{equation}
\end{theorem}

The proof for Theorem \ref{thm:task_fail_prob} considers all sequences $\mathcal{S}_{0}, \dots, \mathcal{S}_{T}$ that correspond to property violation events, and sums over their probabilities under ($\mu_{2},p_{2},\policy_{2}$).

Now let $(\mu_{1},p_{1},\policy_{1}) = (\mu,p,\policybase)$ and $(\mu_{2},p_{2},\policy_{2}) = (\mu,p,\policytask)$; we see that constraining the policy ratio $\frac{\policytask(a|s)}{\policybase(a|s)} \leq \alpha$ for all $a \in \mathcal{A}, s \in \mathcal{S}$ is sufficient to ensure that the bound holds. The total multiplicative increase on the upper bound for property violation going from $\policybase$ to $\policytask$ is thus $\alpha^T$, with the bound being $\eps_{\mathrm{task}} = \eps_{\mathrm{base}} \alpha^T$.

This is a significant result, since we can now provide a prior bound on the probability of property violation for \textit{any} $\policytask$, based \textit{entirely} on the probability of property violation for $\policybase$ and the maximum policy ratio between $\policytask$ and $\policybase$ across all states and actions, with no required knowledge of $\mu$, $p$ or the constraints under which property $\varphi$ holds, so long as initial state distribution and state transition distribution remain the same. Furthermore, by adjusting the value of $\alpha$ we can directly trade off safety guarantees for deviation from the base policy, which can be leveraged to achieve a boost in task-specific performance.

However, note the exponential relationship between $T$ and $\eps_{\mathrm{task}}$; this means that, for even small increases of $\alpha$ from $1$, our prior bound will always eventually explode to the point of becoming trivially $1$ if $T$ is made sufficiently large. Thus, if the prior bound is to be used, SPoRt is best suited to control problems with a low maximum episode length $T$. In practice, however, there are ways to overcome or otherwise mitigate this limitation, as we will see later in Section \ref{case_studies}.

Note that Theorem \ref{thm:task_fail_prob} also provides a bound when $\mu_{2}$ and $p_{2}$ differ from $\mu_{1}$ and $p_{1}$. Thus, our theoretical results can also be applied to \emph{robust control} settings for perturbed systems; see Appendix C.2 for further discussion.


\section{Constraint Satisfaction for a Task Policy}
\label{ensuring_constraint_satisfaction_modified_policy}

For Theorem \ref{thm:task_fail_prob} to hold, we must maintain the hard constraint $\frac{\policytask(a|s)}{\policybase(a|s)} \leq \alpha$ for all $a \in \mathcal{A}, s \in \mathcal{S}$.   
Given a $\policytask$, 
we can achieve this by projecting $\policytask$ onto the feasible set of policy distributions $\mathit{\Pi}_{\alpha,\policybase}$ at each time step: 
\begin{gather}
    {\policy}_{\mathrm{proj}}(a |s) = \mathrm{proj}_{\mathit{\Pi}_{\alpha,\policybase}(s)}({\policytask}(a |s)), \\
    \text{where} \; \mathit{\Pi}_{\alpha,\policybase}(s) =  \left\{\policy : \alpha  \geq \frac{\policy(a |s)}{\policybase(a | s)} \; \forall a \in \mathcal{A} \right\}. 
\end{gather}
While $\policytask$ represents the unconstrained (and potentially unsafe) task-specific policy network that we train or are provided, the projection $\policyproj$ is a policy that we can safely roll out, including during training. Note that $\alpha$ defines the level sets of $\mathit{\Pi}_{\alpha,\policybase}$, which is always non-empty for $\alpha \geq 1$ (since $\policybase$ itself is a valid $\policytask$). Let us now look at how to simplify the computation of this projection step -- we will henceforth assume diagonal Gaussian policies:

\begin{assumption}
\label{asmp:gaussian}
    Both $\policybase$ and $\policytask$ are diagonal Gaussian policies:
        $\policy(\mathbf{a}|s) = \mathcal{N}\left(\mathbf{a};\boldsymbol{\mu},\mathbf{\Sigma}\right)$, where $\mathbf{\Sigma} = \mathrm{diag}\left(\boldsymbol{\sigma}^{2}\right)$,
    and where the policy means $\boldsymbol{\mu}(s)$ and standard deviations $\boldsymbol{\sigma}(s)$ are functions of MDP state and evaluated at each time step using (for example) a policy neural network. 
\end{assumption}

At each time step, we can obtain $\policyproj$ using Theorem \ref{thm:policy_projection}:
\begin{theorem}
\label{thm:policy_projection}
    Assuming diagonal Gaussian policies and using KL divergence as the distance metric for projection, the means and standard deviations of projected policy $\policyproj$ can be computed from $\policybase$ and $\policytask$ by solving the following convex optimization problem at each time step:
    \begin{equation}
        \begin{aligned}
            &\min_{\boldsymbol{\mu}_\mathrm{proj},\boldsymbol{\sigma}_\mathrm{proj}}~ & &J(\boldsymbol{\mu}_\mathrm{proj},\boldsymbol{\sigma}_\mathrm{proj}) \\
            &\mathrm{subject~to}~ & & \prod_{i=1}^{n}\left(\frac{\sigma_{\mathrm{base},i}}{\sigma_{\mathrm{proj},i}} e^{\frac{1}{2} \frac{\left(\mu_{\mathrm{proj},i} - \mu_{\mathrm{base},i}\right)^{2}}{\sigma_{\mathrm{base},i}^{2} - \sigma_{\mathrm{proj},i}^{2}} }\right) \leq \alpha,\\
            & ~ & & 0< \sigma_{\mathrm{proj},i} < \sigma_{\mathrm{base},i} \quad \forall i=1,\dots,n
        \end{aligned}
    \end{equation}
    \begin{align}
        & \text{where} \quad J(\boldsymbol{\mu}_\mathrm{proj},\boldsymbol{\sigma}_\mathrm{proj})\\
        = & \sum_{i=1}^{n} \left( - 2\ln\left(\sigma_{\mathrm{proj},i}\right) + \frac{\sigma_{\mathrm{proj},i}^{2}}{\sigma_{\mathrm{task},i}^{2}} + \frac{\left(\mu_{\mathrm{proj},i} - \mu_{\mathrm{task},i}\right)^{2}}{\sigma_{\mathrm{task},i}^{2}} \right).
    \end{align}
\end{theorem}
It is interesting to note that the standard deviations of $\policyproj$ must all be strictly lower than those of $\policybase$, in other words we require $\policyproj$ to be less exploratory than $\policybase$. This is intuitive considering that we aim to maintain safety by remaining `close' to $\policybase$. We also note that making $\policybase$  more exploratory (with a larger standard deviation) generally results in a larger $\mathit{\Pi}_{\alpha, \policybase}$, allowing for greater policy change and thus task-specific performance boost by $\policyproj$; see Appendix C.3 for details.

We implement and solve this problem using CVXPY \cite{CVXPY_1,CVXPY_2}; on a standard desktop PC the compute time remains in the order of milliseconds for even high-dimensional action spaces, making this method feasible for many RL applications. See Appendix D for implementation details.

\section{Training for Tasks, Under a Bound on Property Violation}
\label{training_while_maintaining_bound_property_violation}

Above, we have shown how to obtain $\policyproj$ from $\policybase$ and $\policytask$. We have implicitly assumed that we already have $\policybase$ and a corresponding bound on probability of property violation $\eps_{\mathrm{base}}$. There are many practical applications where we would also have access to $\policytask$: as an example from robotics, we may have trained $\policytask$ in simulation, where safety is non-critical, but now want to safely deploy this policy on the real robot, for which a tried-and-tested $\policybase$ is known. 

However, other applications may require that we train $\policytask$ to maximize cumulative $r_\mathrm{task}$ while maintaining a prior bound on safety during training, for example if a suitable simulator to train good policies for the real environment is not available. In this case, we would first choose an acceptable $\alpha \geq 1$ by rearranging $\eps_{\mathrm{task}} = \eps_{\mathrm{base}} \alpha^{T} \leq \eps_{\mathrm{max}}$, where $\eps_{\mathrm{max}}$ is a maximum acceptable probability of property violation, and we then learn $\policytask$ (initialized as $\policybase$) while only ever deploying $\policyproj$ during training so as to ensure the bound holds.

Note that there is a distinction between what we train, $\policytask$, and what we actually deploy, $\policyproj$, during training. 
To overcome this, SPoRt uses Projected PPO, outlined in Algorithm \ref{alg:projected_ppo}, to train $\policytask$ for tasks with continuous state-action spaces. The algorithm is inspired by clipped PPO \cite{PPO} but clips the surrogate advantage based on the policy ratio between the new $\policytask$ and previous $\policyproj$, rather than the previous $\policytask$; we store the (log-)probabilities of $\policyproj$ at each time step during data collection to avoid needing to recompute $\policyproj$ during gradient updates. The advantage estimates are also computed using samples collected by deploying $\policyproj$ rather than $\policytask$.

\begin{algorithm}[tb]
    \caption{Projected PPO}\label{alg:projected_ppo}
    \textbf{Input}: $\theta_{\mathrm{base}}$, $\alpha$, $T$
    \begin{algorithmic}[1]
        \STATE Obtain initial critic parameters $\phi_{0}$ by warm-starting the critic using PPO (episode length $T$) with $\policybase$
        \STATE Initial task-specific policy parameters $\theta_{0} \leftarrow \theta_{\mathrm{base}}$
        \FOR{$k = 0,1,2,\dots$}
            \STATE{Collect trajectories $\mathcal{D}_k = \{ (\traj{}{T})_{i} \}$ by running projected policy $\policy_{\mathrm{proj},\theta_{k}} = \mathrm{proj}_{\Pi_{\alpha,\policybase}}(\policy_{\mathrm{task},\theta_{k}})$ in the environment. Store  $\policy_{\mathrm{proj},\theta_{k}}(\cdot|\st) \; \forall \st \in \traj{}{T}, \traj{}{T} \in \mathcal{D}_{k}$}
            \STATE Compute rewards-to-go $\hat{R}_{t}$
            \STATE{Compute advantage estimates $\hat{A}_{t}$ using $V_{\phi_{k}}$}
            \STATE{Update task-specific policy:
            \begin{gather*}
                \theta_{k+1} = \mathrm{arg} \max_{\theta} \frac{1}{|\mathcal{D}_{k}|T} \sum_{\traj{}{} \in \mathcal{D}_{k}} \sum_{t=0}^{T} \min \Bigl\{ \\
                \hspace{-1.2em} \frac{\policy_{\mathrm{task}, \theta}(\at|\st)}{\policy_{\mathrm{proj},\theta_{k}}(\at|\st)} A^{\policy_{\mathrm{proj},\theta_{k}}}(\st,\at),\; g\bigl(\xi,A^{\policy_{\mathrm{proj},\theta_{k}}}(\st,\at)\bigr) \Bigr\} \\
                \mathrm{where} \; g\left(\xi,A\right) =
                \begin{cases}
                    (1 + \xi)A & \text{if } A \geq 0 \\
                    (1 - \xi)A & \text{if } A < 0 \\
                \end{cases}
            \end{gather*}}
            \STATE{Fit value function:
            \begin{equation}
                \phi_{k+1} = \mathrm{arg} \min_{\phi} \frac{1}{|\mathcal{D}_{k}|T} \sum_{\traj{}{} \in \mathcal{D}_{k}} \sum_{t=0}^{T} \left( V_\phi(\st) - \hat{R}_{t}\right)^{2}
            \end{equation}}
        \ENDFOR        
    \end{algorithmic}
\end{algorithm}

The clipping sets the gradient to zero beyond the maximum/minimum allowed policy ratio of $\policytask$ to $\policyproj$, preventing $\policytask$ from drifting away beyond clipping ratio $\xi$ of $\policyproj$, in this way, we maintain an acceptable amount of mismatch between $\policytask$ and $\policyproj$ and stop $\policytask$ from drifting away from the feasible set of allowed policy distributions.

We also warm-start the value function network for $r_\mathrm{task}$ before training $\policytask$, since an accurate value function is important for effective fine-tuning. This is done by training the value function using clipped PPO until convergence while keeping the policy network weights fixed as those of $\policybase$.

\section{SPoRt in Action: Case Studies}
\label{case_studies}

We apply SPoRt to a \emph{reach-avoid} property, wherein the agent must reach a goal within a time limit while avoiding collision with a hazard up until the goal is reached. Such an objective is standard within the control and verification literature, and indeed can be used to model many real-world problems. For completeness, we provide the LTL formula and robustness metric for the time-bounded reach-avoid property in Appendix E.1, with a reminder that SPoRt can be used over general LTL specifications. 

We implemented the environment in Safety Gymnasium \cite{Ji2023} using a point agent and with the goal and hazard being green and red circular regions, respectively (cf.  Figure \ref{fig:trajectories} and \ref{fig:episode:constrained_episode}). This setup was chosen since it allows for easy interpretability of results while remaining a reasonable abstraction of a real robotic navigation task using a skid-steering mobile robot with a LiDAR sensor; a description of the MDP observation and action spaces can be found in Appendix E.2. The episode is reset if the agent enters the hazard or goal sets, or if the maximum episode length is exceeded. To mitigate the exponential relationship between maximum episode length $T$ and bound $\eps_{\mathrm{task}}$, we reduced the control frequency ten-fold from the default 
(up to $100$ simulation steps per environment step), in-keeping with the observation in Section \ref{probability_property_violation} that
SPoRt is best suited to control problems with low $T$.

\begin{figure*}[t]
    \centering
    \subfloat[]{\includegraphics[width=0.987\linewidth]{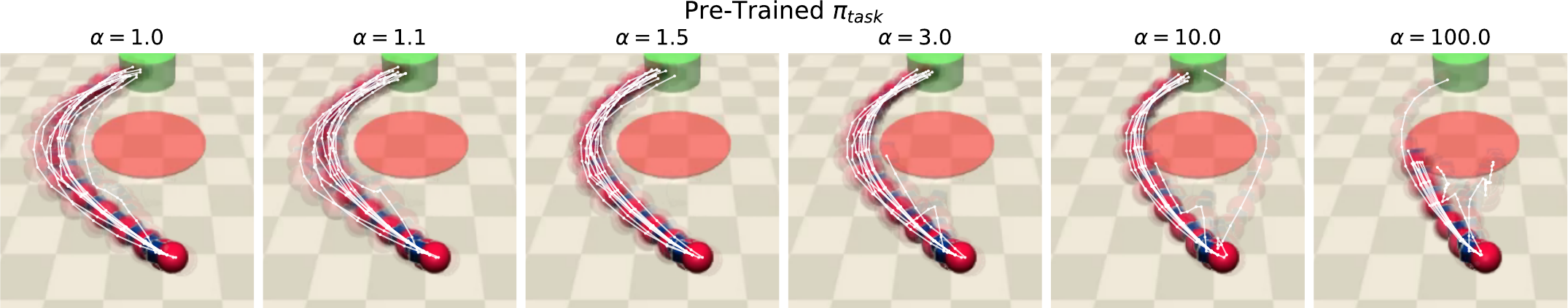}
    \label{fig:trajectories:pretrained_trajectories}}
    \hfil
    \subfloat[]{\includegraphics[width=0.987\linewidth]{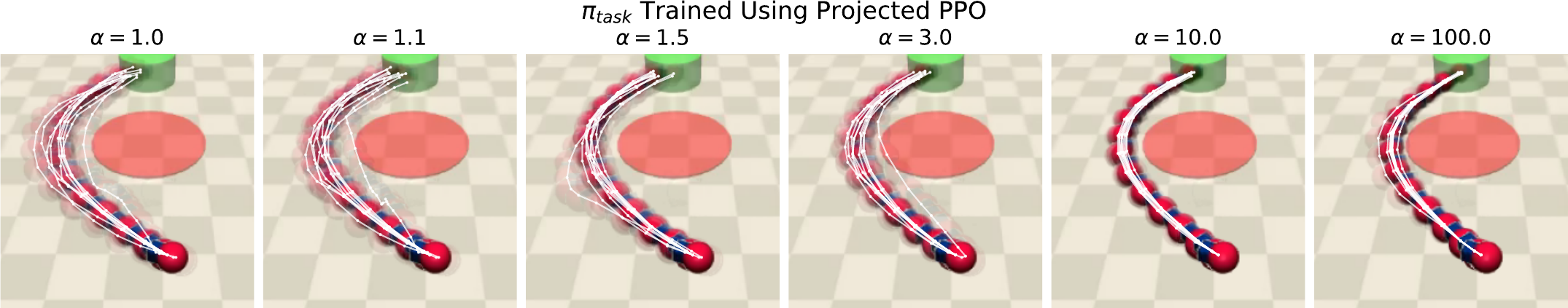}
    \label{fig:trajectories:constrained_trajectories}}
    \caption{Sample distributions of episode trajectories from the reach-avoid experiment using $\policyproj$ for different values of $\alpha$. (\ref{fig:trajectories:pretrained_trajectories}) Case 1 (pre-trained $\policytask$). (\ref{fig:trajectories:constrained_trajectories}) Case 2 ($\policytask$ trained using Projected PPO). Action seeding for each episode was controlled across different values of $\alpha$ and across the different cases, so all results depend on $\alpha$ and the training of $\policytask$.}
    \label{fig:trajectories}
\end{figure*}

\begin{figure*}[t]
    \centering
    \includegraphics[width=0.960\linewidth]{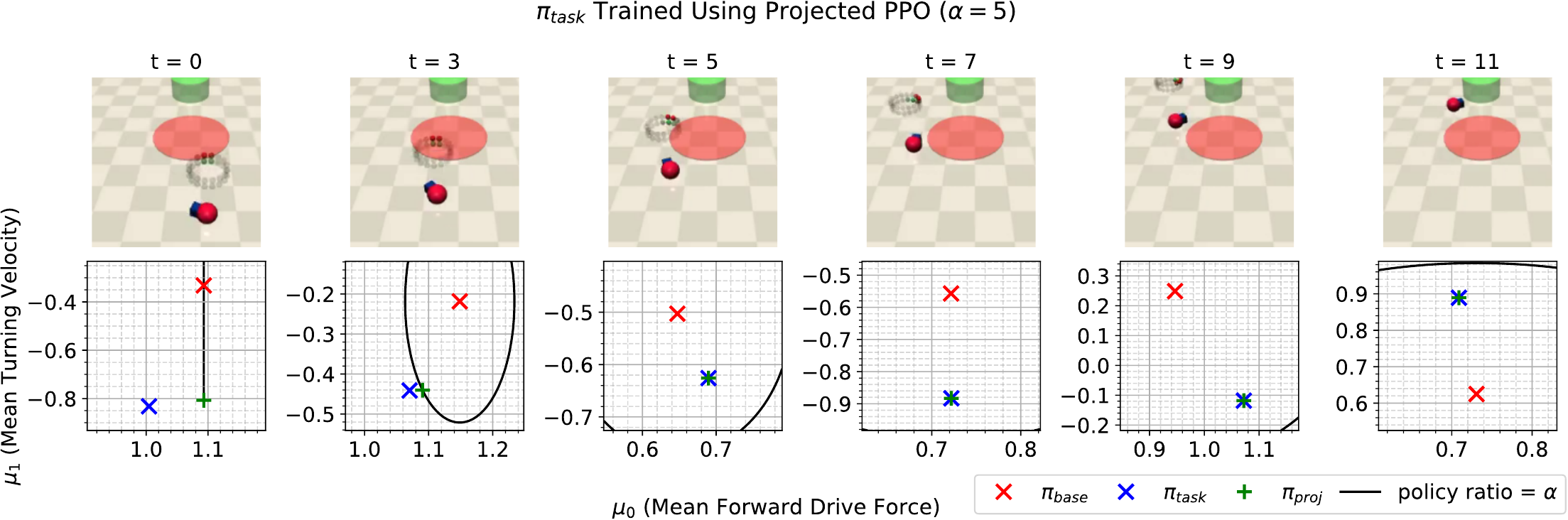}
    \caption{Snapshots across an example episode of the reach-avoid experiment using $\policyproj$ for Case 2 ($\policytask$ trained using Projected PPO) and $\alpha = 5$; the bottom plots present the action means at the corresponding time step, with the black contour depicting the $\alpha = 5$ level set. Note that positive mean turning velocity represents anticlockwise rotation. The halo above the agent is a visualization of its LiDAR observations for the hazard and goal. See Appendix E.5 for Case 1 (pre-trained $\policytask$).}
    \label{fig:episode:constrained_episode}
\end{figure*}

$\policybase$ was trained so as to achieve a high probability of satisfying the property (reach-avoid), while remaining fairly exploratory. This was achieved by training $\policybase$ using Soft Actor-Critic (SAC) \cite{SAC} with a \emph{sparse} reward scheme (corresponding to property satisfaction across an \emph{entire} episode). Alternative synthesis schemes are possible. Further details on training $\policybase$ can be found in Appendix E.3. Once trained, around $N = 10000$ scenarios were collected to determine $\eps_{\mathrm{base}}$ with high confidence ($\beta = \num{1e-7}$, see \cite{Campi2018}) using Corollary \ref{cor:constraint_removal}. Note that while training $\policybase$ the maximum episode length was set to $T = 100$ yet by the end of training the average episode length was much lower, at around $T = 14$. Thus, to keep the value of $\eps_{\mathrm{task}}=\eps_{\mathrm{base}} \alpha^{T}$ as low as possible, the maximum episode length was reduced to $T = 21$ after training $\policybase$ (with $\eps_{\mathrm{base}}$ computed using scenarios of this length). Further discussions (including how SPoRt can be modified to do this automatically) can be found in Appendix E.4.

For our case studies we trained $\policytask$ to reach the goal as quickly as possible: accordingly,  $r_\mathrm{task}$ was the standard dense reward for reaching a goal used by Safety Gymnasium. Notice that the set task (and corresponding reward) clearly leads to a potential violation of the property (reach-avoid) of interest.
We consider two separate cases, as follows: 
\paragraph{Case 1: Pre-Trained Task Policy.} $\policytask$ is trained separately without any consideration for property violation. As a result, under $\policytask$ the agent quickly drives directly towards the goal with no hazard avoidance. This represents applications where $\policytask$ has been pre-trained in an environment where safety is not critical (for example in a robotics simulator) and we want to safely test it on the real environment (see Section \ref{ensuring_constraint_satisfaction_modified_policy}).

\paragraph{Case 2: Task Policy Trained Using Projected PPO.} This represents applications where we have $\policybase$ and $\eps_{\mathrm{base}}$ (as from above) and now want to fine-tune our policy to be faster (thus obtaining $\policytask$), while maintaining an acceptable given bound on property violation (see Section \ref{training_while_maintaining_bound_property_violation}).

Note that we use the same $\policybase$ for both cases.

\section{SPoRt Report: Results and Discussion}
\label{results_discussion}

Both cases were tested for 1000 episodes at different values of $\alpha$ (such that $\policy_{\mathrm{proj}} = \mathrm{proj}_{\Pi_{\alpha,\policybase}}(\policy_{\mathrm{task}})$), ranging from $\alpha = 1$ (i.e. $\policyproj = \policybase$) to the point where the empirical violation rate exceeded a threshold. For Case 2, $\policytask$ was trained until convergence at each value of $\alpha$, prior to testing. Further details on training $\policytask$ for both cases can be found in Appendix E.3. Action seeding for each episode was controlled across different values of $\alpha$ and across the different cases, so all results depend on $\alpha$ and the training of $\policytask$. From our results we seek to answer the following questions (Qs):
\begin{enumerate}[itemsep=0.01cm, topsep=0.1cm]
    \item Does increasing $\alpha$ trade off safety for performance?
    \item How does performance compare between Case 1 and 2?
    \item How conservative is the prior bound $\eps_{\mathrm{task}}=\eps_{\mathrm{base}} \alpha^{T}$?
\end{enumerate}

\begin{figure*}[t]
    \centering
    \subfloat[]{\includegraphics[width=0.474\linewidth]{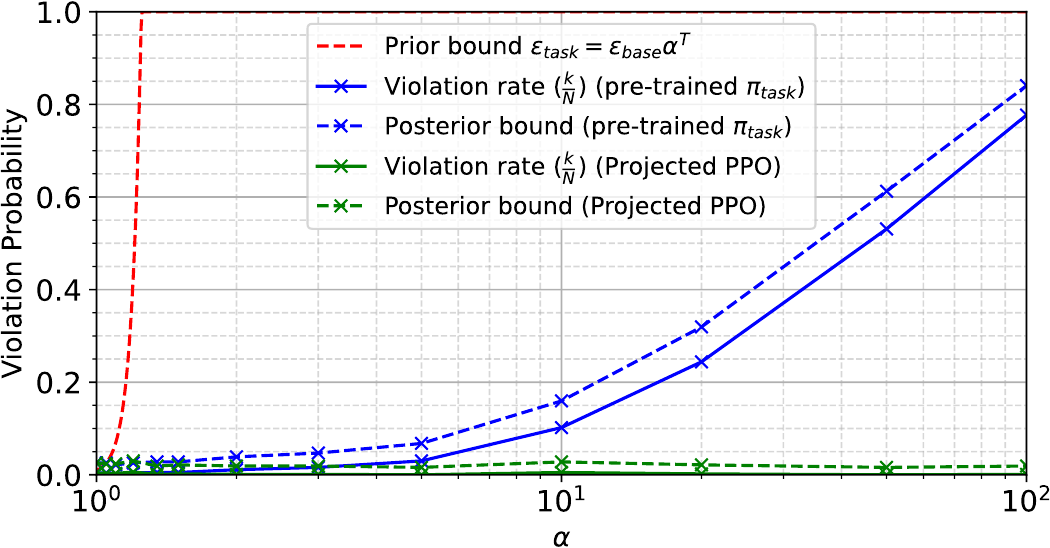}
    \label{fig:plots:failure_probs}}
    \hfil
    \subfloat[]{\includegraphics[width=0.514\linewidth]{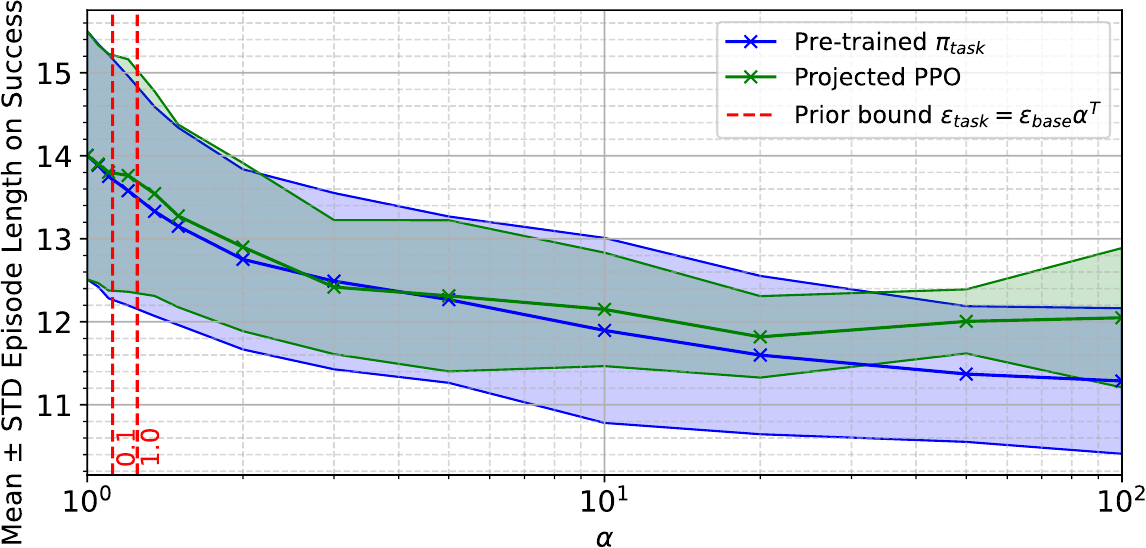}
    \label{fig:plots:mean_std_time_taken}}
    \caption{Results from the reach-avoid experiment for both Case 1 (pre-trained $\policytask$) and 2 ($\policytask$ trained using Projected PPO). (\ref{fig:plots:failure_probs}) Violation probabilities over different values of $\alpha$. (\ref{fig:plots:mean_std_time_taken}) Mean and standard deviation episode length for successful trajectories for different values of $\alpha$. Action seeding for each episode was controlled across different values of $\alpha$ and across the different cases, so all results depend on $\alpha$ and the training of $\policytask$. The same figures but zoomed in to the scale across which $\eps_{\mathrm{task}} \leq 1$ can be found in Appendix E.5.}
    \label{fig:plots}
\end{figure*}

Figure \ref{fig:trajectories} presents sample distributions of episode trajectories for both cases over different values of $\alpha$. In both cases, we see that as $\alpha$ increases, the trajectories bend more tightly around the hazard, suggesting a reduction in action variance,\footnote{Recall we work with diagonal Gaussian policies, hence the consideration of action mean and variance.} as well as an action mean that takes the agent closer to the hazard. In fact, in Case 1 for $\alpha = 100$, the agent's mean trajectory crosses the hazard. Thus increasing $\alpha$ is shown to trade off safety for performance, answering Q1.

However, we can also appreciate the difference between Case 1 and Case 2: while Case 1 produces an action mean that drives the agent through the hazard for $\alpha = 100$, Case 2 instead produces a more reduced action variance, while retaining an action mean that keeps the agent outside the hazard, as expected. Thus, to answer Q2, we see that Case 2 allows us to provide better performance, whilst retaining a `better behaved' (and indeed `safe') $\policyproj$ compared to Case 1; accordingly, we argue that since training $\policytask$ using Projected PPO deploys $\policyproj$ during training, $\policytask$ learns to optimize performance of $\policyproj$ compared to na\"{i}vely training $\policytask$ a priori with no consideration of how $\policyproj$ will perform.

Figure \ref{fig:episode:constrained_episode} presents a more detailed view of the agent behavior over an example episode for Case 2, for $\alpha = 5$ (representing a compromise between safety and performance). Looking at mean turning velocity over the episode, we see that while both $\policybase$ and $\policytask$ drive the agent clockwise around the hazard, $\policytask$ induces sharper turning, taking the agent closer to the hazard and drawing a tighter, shorter curve while maintaining the same or faster forward drive force. However, this sharper turning is constrained such that $\policyproj$ always lies within the $\alpha = 5$ level set (see Section \ref{ensuring_constraint_satisfaction_modified_policy}). Note at $\policytask$ applies a reduced forward drive force at the very start of the episode compared to $\policybase$, which makes sense given the agent is initially pointing away from the goal, so $\policytask$ reduces episode length by first pointing the agent closer to the agent before driving forward. Appendix E.5 provides a similar analysis for Case 1.

Figure \ref{fig:plots:failure_probs} presents violation probabilities over different values of $\alpha$ for both cases. The most striking observation is the conservativeness of prior bound $\eps_{\mathrm{task}}$, 
which grows exponentially from $\eps_{\mathrm{base}} = 0.009$ at $\alpha = 1$ to $\eps_{\mathrm{task}} = 1$ at around $\alpha = 1.25$ (beyond which point the bound is no longer useful), yet the posterior bounds on property violation (obtained by applying Corollary \ref{cor:constraint_removal} to the $N=1000$ test samples) remain at around $0.025$ over this range. We do also see exponential growth in the posterior bound for Case 1, but this happens over a completely different scale ($\alpha = 1$ to $100$ rather than $1$ to $1.25$). The posterior bound for Case 2 remains at $0.025$ for even $\alpha = 100$, suggesting much safer behavior compared to Case 1 for the same $\alpha$.

Figure \ref{fig:plots:mean_std_time_taken} presents the mean and standard deviation episode length for successful trajectories over different values of $\alpha$ for both cases. For both we see a similar reduction in mean and standard deviation as $\alpha$ increases until around $\alpha = 10$, at which point the mean plateaus (to around 12.0 at $\alpha = 100$, for $14.3\%$ total reduction) but standard deviation shrinks for Case 2 while the mean continues to decrease for Case 1 (to around 11.3 at $\alpha = 100$, for $19.3\%$ total reduction)); this comes at the cost of substantially increased violation rate for Case 1, shown by Figure \ref{fig:plots:failure_probs}. Another important observation is that while we know from Figure \ref{fig:plots:failure_probs} that $\eps_{\rm{task}}$ is very conservative, we do see a measurable ($2.1\%$) reduction in mean episode length for both cases from 14.0 at $\alpha = 1$ ($\eps_{\rm{task}} = 0.009$) to around 13.7 at $\alpha = 1.12$ ($\eps_{\rm{task}} = 0.1$, a fairly sensible (if high) value). Appendix E.5 provides the same figures but zoomed in to the scale across which $\eps_{\mathrm{task}} \leq 1$.

These observations further confirm our earlier answers to Q1 and Q2, whilst now we also have an answer for Q3: the prior bound can be very conservative, though it is possible to see measurable improvement in performance while the bound remains fairly sensible.

\section{Limits to Sporting SPoRt}
\label{limitations}

The most obvious limitation of SPoRt is the conservativeness of the prior bound $\eps_{\mathrm{task}}=\eps_{\mathrm{base}} \alpha^{T}$, which prevents significant policy changes if the bound is to be used to guarantee safety. This conservativeness also results in the limitation of needing $T$ to be as low as possible, making SPoRt unsuitable for applications where $T$ is high (though we have seen ways to mitigate this limitation). Another limitations include the reliance on collecting many scenarios to obtain a useful bound $\eps_{\mathrm{base}}$, which may not be practical for some applications, as well as the requirement of stochastic policies (and, ideally, a fairly exploratory $\policybase$ to achieve noticeable policy change). Despite these theoretical limits, we have displayed the usefulness of the end-to-end architecture of SPoRt in meaningful simulation studies, which are promising for upcoming real-world implementations of SPoRt.

\section{Conclusions}
\label{conclusions}

We have presented novel theoretical results that provide a prior bound on the probability of (safety) property violation for a task-specific policy in a model-free, episodic RL setup, based on a new `maximum policy ratio' established vis-a-vis a given `safe' base policy. Based on these bounds, we have presented an end-to-end architecture, SPoRt, which combines a data-driven approach for obtaining such a bound for the base policy with a projection-based approach for training the task-specific policy while maintaining a user-specified prior bound on (safety) property violation, thus trading off safety guarantees and task-specific performance. In view of promising experimental simulation results,  future work will focus on reducing the conservativeness of the prior bound, to improve its utility in practical real-world applications. 

\vfill


\section*{Acknowledgments}
This work was supported by the EPSRC Centre for Doctoral Training in Autonomous Intelligent Machines and Systems [EP/S024050/1].

\bibliographystyle{named}
\bibliography{references}

\begin{thebibliography}{}

\bibitem[\protect\citeauthoryear{Achiam \bgroup \em et al.\egroup }{2017}]{Achiam2017}
Joshua Achiam, David Held, Aviv Tamar, and Pieter Abbeel.
\newblock Constrained policy optimization.
\newblock In {\em Proceedings of the 34th International Conference on Machine Learning}, volume~70 of {\em Proceedings of Machine Learning Research}, pages 22--31. PMLR, 06--11 Aug 2017.

\bibitem[\protect\citeauthoryear{Agrawal \bgroup \em et al.\egroup }{2018}]{CVXPY_2}
Akshay Agrawal, Robin Verschueren, Steven Diamond, and Stephen Boyd.
\newblock A rewriting system for convex optimization problems.
\newblock {\em Journal of Control and Decision}, 5(1):42--60, 2018.

\bibitem[\protect\citeauthoryear{Alshiekh \bgroup \em et al.\egroup }{2018}]{Alshiekh2018}
Mohammed Alshiekh, Roderick Bloem, Rüdiger Ehlers, Bettina Könighofer, Scott Niekum, and Ufuk Topcu.
\newblock Safe reinforcement learning via shielding.
\newblock {\em Proceedings of the AAAI Conference on Artificial Intelligence}, 32(1), April 2018.

\bibitem[\protect\citeauthoryear{Altman}{2021}]{Altman2021}
Eitan Altman.
\newblock {\em Constrained Markov Decision Processes: Stochastic Modeling}.
\newblock Routledge, Boca Raton, 1 edition, December 2021.

\bibitem[\protect\citeauthoryear{Calafiore and Campi}{2006}]{Calafiore2006}
G.C. Calafiore and M.C. Campi.
\newblock The scenario approach to robust control design.
\newblock {\em {IEEE} Transactions on Automatic Control}, 51(5):742--753, May 2006.

\bibitem[\protect\citeauthoryear{Campi and Garatti}{2008}]{Campi2008}
M.~C. Campi and S.~Garatti.
\newblock The exact feasibility of randomized solutions of uncertain convex programs.
\newblock {\em SIAM Journal on Optimization}, 19(3):1211--1230, 2008.

\bibitem[\protect\citeauthoryear{Campi and Garatti}{2010}]{Campi_2010}
M.~C. Campi and S.~Garatti.
\newblock A sampling-and-discarding approach to chance-constrained optimization: Feasibility and optimality.
\newblock {\em Journal of Optimization Theory and Applications}, 148(2):257–280, October 2010.

\bibitem[\protect\citeauthoryear{Campi and Garatti}{2018}]{Campi2018}
Marco~C Campi and Simone Garatti.
\newblock {\em {Introduction to the Scenario Approach}}.
\newblock Society for Industrial and Applied Mathematics, Philadelphia, PA, November 2018.

\bibitem[\protect\citeauthoryear{Chow \bgroup \em et al.\egroup }{2018}]{Chow2018}
Yinlam Chow, Ofir Nachum, Edgar Duenez-Guzman, and Mohammad Ghavamzadeh.
\newblock A {Lyapunov}-based {Approach} to {Safe} {Reinforcement} {Learning}.
\newblock In {\em Advances in {Neural} {Information} {Processing} {Systems}}, volume~31. Curran Associates, Inc., 2018.

\bibitem[\protect\citeauthoryear{Chu \bgroup \em et al.\egroup }{2020}]{Chu2020}
Tianshu Chu, Jie Wang, Lara Codeca, and Zhaojian Li.
\newblock Multi-{Agent} {Deep} {Reinforcement} {Learning} for {Large}-{Scale} {Traffic} {Signal} {Control}.
\newblock {\em IEEE Transactions on Intelligent Transportation Systems}, 21(3):1086--1095, March 2020.

\bibitem[\protect\citeauthoryear{Diamond and Boyd}{2016}]{CVXPY_1}
Steven Diamond and Stephen Boyd.
\newblock {CVXPY}: {A} {P}ython-embedded modeling language for convex optimization.
\newblock {\em Journal of Machine Learning Research}, 17(83):1--5, 2016.

\bibitem[\protect\citeauthoryear{Haarnoja \bgroup \em et al.\egroup }{2018}]{SAC}
Tuomas Haarnoja, Aurick Zhou, Pieter Abbeel, and Sergey Levine.
\newblock Soft actor-critic: Off-policy maximum entropy deep reinforcement learning with a stochastic actor.
\newblock In {\em Proceedings of the 35th International Conference on Machine Learning, {ICML} 2018, Stockholmsm{\"{a}}ssan, Stockholm, Sweden, July 10-15, 2018}, volume~80 of {\em Proceedings of Machine Learning Research}, pages 1856--1865. {PMLR}, 2018.

\bibitem[\protect\citeauthoryear{Hasanbeig \bgroup \em et al.\egroup }{2023}]{LCRL}
Hosein Hasanbeig, Daniel Kroening, and Alessandro Abate.
\newblock Certified reinforcement learning with logic guidance.
\newblock {\em Artificial Intelligence}, 322:103949, September 2023.

\bibitem[\protect\citeauthoryear{Hsu \bgroup \em et al.\egroup }{2024}]{Hsu2024}
Kai-Chieh Hsu, Haimin Hu, and Jaime~F. Fisac.
\newblock The safety filter: A unified view of safety-critical control in autonomous systems.
\newblock {\em Annual Review of Control, Robotics, and Autonomous Systems}, 7(1):47–72, July 2024.

\bibitem[\protect\citeauthoryear{Hwangbo \bgroup \em et al.\egroup }{2019}]{Hwangbo2019}
Jemin Hwangbo, Joonho Lee, Alexey Dosovitskiy, Dario Bellicoso, Vassilios Tsounis, Vladlen Koltun, and Marco Hutter.
\newblock Learning agile and dynamic motor skills for legged robots.
\newblock {\em Science Robotics}, 4(26):eaau5872, January 2019.

\bibitem[\protect\citeauthoryear{Isele \bgroup \em et al.\egroup }{2018}]{Isele2018}
David Isele, Alireza Nakhaei, and Kikuo Fujimura.
\newblock Safe {Reinforcement} {Learning} on {Autonomous} {Vehicles}.
\newblock In {\em 2018 {IEEE}/{RSJ} {International} {Conference} on {Intelligent} {Robots} and {Systems} ({IROS})}, pages 1--6, Madrid, October 2018. IEEE.

\bibitem[\protect\citeauthoryear{Ji \bgroup \em et al.\egroup }{2023}]{Ji2023}
Jiaming Ji, Borong Zhang, Jiayi Zhou, Xuehai Pan, Weidong Huang, Ruiyang Sun, Yiran Geng, Yifan Zhong, Josef Dai, and Yaodong Yang.
\newblock Safety gymnasium: A unified safe reinforcement learning benchmark.
\newblock In {\em Thirty-seventh Conference on Neural Information Processing Systems Datasets and Benchmarks Track}, 2023.

\bibitem[\protect\citeauthoryear{Kober and Peters}{2014}]{Kober2014}
Jens Kober and Jan Peters.
\newblock {\em Reinforcement Learning in Robotics: A Survey}, volume~97 of {\em Springer Tracts in Advanced Robotics}, page 9–67.
\newblock Springer International Publishing, Cham, 2014.

\bibitem[\protect\citeauthoryear{Konighofer \bgroup \em et al.\egroup }{2023}]{Konighofer2023}
Bettina Konighofer, Julian Rudolf, Alexander Palmisano, Martin Tappler, and Roderick Bloem.
\newblock Online shielding for reinforcement learning.
\newblock {\em Innovations in Systems and Software Engineering}, 19(4):379–394, December 2023.

\bibitem[\protect\citeauthoryear{Lee \bgroup \em et al.\egroup }{2023}]{Lee2023}
Hyosun Lee, Yohee Han, and Youngchan Kim.
\newblock Reinforcement learning for traffic signal control: {Incorporating} a virtual mesoscopic model for depicting oversaturated traffic conditions.
\newblock {\em Engineering Applications of Artificial Intelligence}, 126:107005, November 2023.

\bibitem[\protect\citeauthoryear{Li \bgroup \em et al.\egroup }{2022}]{Li2022}
Guofa Li, Yifan Yang, Shen Li, Xingda Qu, Nengchao Lyu, and Shengbo~Eben Li.
\newblock Decision making of autonomous vehicles in lane change scenarios: {Deep} reinforcement learning approaches with risk awareness.
\newblock {\em Transportation Research Part C: Emerging Technologies}, 134:103452, January 2022.

\bibitem[\protect\citeauthoryear{Ma \bgroup \em et al.\egroup }{2021}]{Ma2021}
Xiaobai Ma, Jiachen Li, Mykel~J. Kochenderfer, David Isele, and Kikuo Fujimura.
\newblock Reinforcement {Learning} for {Autonomous} {Driving} with {Latent} {State} {Inference} and {Spatial}-{Temporal} {Relationships}.
\newblock In {\em 2021 {IEEE} {International} {Conference} on {Robotics} and {Automation} ({ICRA})}, pages 6064--6071, Xi'an, China, May 2021. IEEE.

\bibitem[\protect\citeauthoryear{Mason and Grijalva}{2019}]{Mason2019}
Karl Mason and Santiago Grijalva.
\newblock A review of reinforcement learning for autonomous building energy management.
\newblock {\em Computers \& Electrical Engineering}, 78:300--312, 2019.

\bibitem[\protect\citeauthoryear{Milosevic \bgroup \em et al.\egroup }{2024}]{Milosevic2024}
Nikola Milosevic, Johannes Müller, and Nico Scherf.
\newblock Embedding {Safety} into {RL}: {A} {New} {Take} on {Trust} {Region} {Methods}.
\newblock {\em arXiv}, November 2024.

\bibitem[\protect\citeauthoryear{Pnueli}{1977}]{Pnueli_1977}
Amir Pnueli.
\newblock The temporal logic of programs.
\newblock In {\em 18th Annual Symposium on Foundations of Computer Science (sfcs 1977)}, page 46–57, Providence, RI, USA, September 1977. IEEE.

\bibitem[\protect\citeauthoryear{Schulman \bgroup \em et al.\egroup }{2015}]{Schulman2015}
John Schulman, Sergey Levine, Pieter Abbeel, Michael Jordan, and Philipp Moritz.
\newblock Trust region policy optimization.
\newblock In {\em Proceedings of the 32nd International Conference on Machine Learning}, volume~37 of {\em Proceedings of Machine Learning Research}, pages 1889--1897, Lille, France, 07--09 Jul 2015. PMLR.

\bibitem[\protect\citeauthoryear{Schulman \bgroup \em et al.\egroup }{2017}]{PPO}
John Schulman, Filip Wolski, Prafulla Dhariwal, Alec Radford, and Oleg Klimov.
\newblock Proximal policy optimization algorithms.
\newblock {\em arXiv}, 2017.

\bibitem[\protect\citeauthoryear{Singh \bgroup \em et al.\egroup }{2022}]{Singh_Kumar_Singh_2022}
Bharat Singh, Rajesh Kumar, and Vinay~Pratap Singh.
\newblock Reinforcement learning in robotic applications: a comprehensive survey.
\newblock {\em Artificial Intelligence Review}, 55(2):945–990, February 2022.

\bibitem[\protect\citeauthoryear{Soudjani and Abate}{2012}]{bcSA12}
S.~Esmaeil~Zadeh Soudjani and A.~Abate.
\newblock Higher order approximations for verification of stochastic hybrid systems.
\newblock In {\em Proceedings of ATVA12, LNCS 7561}, pages 416--434. Springer Verlag, 2012.

\bibitem[\protect\citeauthoryear{Soudjani and Abate}{2013}]{additive_difference}
Sadegh Esmaeil~Zadeh Soudjani and Alessandro Abate.
\newblock Adaptive and sequential gridding procedures for the abstraction and verification of stochastic processes.
\newblock {\em SIAM Journal on Applied Dynamical Systems}, 12(2):921–956, January 2013.

\bibitem[\protect\citeauthoryear{Soudjani and Abate}{2015}]{additive_difference_2}
Sadegh Esmaeil~Zadeh Soudjani and Alessandro Abate.
\newblock Quantitative approximation of the probability distribution of a markov process by formal abstractions.
\newblock {\em Logical Methods in Computer Science}, Volume 11, Issue 3:1584, September 2015.

\bibitem[\protect\citeauthoryear{Sutton and Barto}{2014}]{Sutton_Barto_2014}
Richard~S. Sutton and Andrew Barto.
\newblock {\em Reinforcement learning: an introduction}.
\newblock Adaptive computation and machine learning. The MIT Press, Cambridge, Massachusetts, nachdruck edition, 2014.

\bibitem[\protect\citeauthoryear{Tempo \bgroup \em et al.\egroup }{2005}]{Tempo2005}
R~Tempo, Giuseppe Calafiore, and Fabrizio Dabbene.
\newblock {\em Randomized algorithms for analysis and control of uncertain systems}.
\newblock Communications and control engineering series. Springer, London, 2005.

\bibitem[\protect\citeauthoryear{Vertovec and Margellos}{2023}]{Vertovec2023}
Nikolaus Vertovec and Kostas Margellos.
\newblock State {Aggregation} for {Distributed} {Value} {Iteration} in {Dynamic} {Programming}.
\newblock {\em IEEE Control Systems Letters}, 7:2269--2274, 2023.

\bibitem[\protect\citeauthoryear{Vertovec \bgroup \em et al.\egroup }{2024}]{Vertovec2024}
Nikolaus Vertovec, Kostas Margellos, and Maria Prandini.
\newblock Finite sample learning of moving targets.
\newblock {\em arXiv}, August 2024.

\bibitem[\protect\citeauthoryear{Vidyasagar}{2003}]{Vidyasagar2003}
M.~Vidyasagar.
\newblock {\em Learning and Generalisation}.
\newblock Springer London, 2003.

\end{thebibliography}


\begin{thebibliography}{10}

\bibitem{Maler_Nickovic_2004}
Oded Maler and Dejan Nickovic.
\newblock {\em Monitoring Temporal Properties of Continuous Signals}, volume 3253 of {\em Lecture Notes in Computer Science}, page 152–166.
\newblock Springer Berlin Heidelberg, Berlin, Heidelberg, 2004.

\bibitem{Fainekos_Pappas_2009}
Georgios~E. Fainekos and George~J. Pappas.
\newblock Robustness of temporal logic specifications for continuous-time signals.
\newblock {\em Theoretical Computer Science}, 410(42):4262–4291, September 2009.

\bibitem{Madsen_Vaidyanathan_Sadraddini_Vasile_DeLateur_Weiss_Densmore_Belta_2018}
Curtis Madsen, Prashant Vaidyanathan, Sadra Sadraddini, Cristian-Ioan Vasile, Nicholas~A. DeLateur, Ron Weiss, Douglas Densmore, and Calin Belta.
\newblock Metrics for signal temporal logic formulae.
\newblock In {\em 2018 IEEE Conference on Decision and Control (CDC)}, page 1542–1547, Miami Beach, FL, December 2018. IEEE.

\bibitem{LCRL}
Hosein Hasanbeig, Daniel Kroening, and Alessandro Abate.
\newblock Certified reinforcement learning with logic guidance.
\newblock {\em Artificial Intelligence}, 322:103949, September 2023.

\bibitem{Sutton_Barto_2014}
Richard~S. Sutton and Andrew Barto.
\newblock {\em Reinforcement learning: an introduction}.
\newblock Adaptive computation and machine learning. The MIT Press, Cambridge, Massachusetts, nachdruck edition, 2014.

\bibitem{Pnueli_1977}
Amir Pnueli.
\newblock The temporal logic of programs.
\newblock In {\em 18th Annual Symposium on Foundations of Computer Science (sfcs 1977)}, page 46–57, Providence, RI, USA, September 1977. IEEE.

\bibitem{additive_difference}
Sadegh Esmaeil~Zadeh Soudjani and Alessandro Abate.
\newblock Adaptive and sequential gridding procedures for the abstraction and verification of stochastic processes.
\newblock {\em SIAM Journal on Applied Dynamical Systems}, 12(2):921–956, January 2013.

\bibitem{additive_difference_2}
Sadegh Esmaeil~Zadeh Soudjani and Alessandro Abate.
\newblock Quantitative approximation of the probability distribution of a markov process by formal abstractions.
\newblock {\em Logical Methods in Computer Science}, Volume 11, Issue 3:1584, September 2015.

\bibitem{scenario_approach}
Marco Campi and Simone Garatti.
\newblock {\em Introduction to the scenario approach}.
\newblock MOS-SIAM series on optimization. Society for Industrial and Applied Mathematics: Mathematical Optimization Society, Philadelphia, 2018.

\bibitem{CVXPY_1}
Steven Diamond and Stephen Boyd.
\newblock {CVXPY}: {A} {P}ython-embedded modeling language for convex optimization.
\newblock {\em Journal of Machine Learning Research}, 17(83):1--5, 2016.

\bibitem{CVXPY_2}
Akshay Agrawal, Robin Verschueren, Steven Diamond, and Stephen Boyd.
\newblock A rewriting system for convex optimization problems.
\newblock {\em Journal of Control and Decision}, 5(1):42--60, 2018.

\bibitem{Ji2023}
Jiaming Ji, Borong Zhang, Jiayi Zhou, Xuehai Pan, Weidong Huang, Ruiyang Sun, Yiran Geng, Yifan Zhong, Josef Dai, and Yaodong Yang.
\newblock Safety gymnasium: A unified safe reinforcement learning benchmark.
\newblock In {\em Thirty-seventh Conference on Neural Information Processing Systems Datasets and Benchmarks Track}, 2023.

\bibitem{Bengio2009}
Yoshua Bengio, Jérôme Louradour, Ronan Collobert, and Jason Weston.
\newblock Curriculum learning.
\newblock In {\em Proceedings of the 26th Annual International Conference on Machine Learning}, page 41–48, Montreal Quebec Canada, June 2009. ACM.

\bibitem{Schaul2016}
Tom Schaul, John Quan, Ioannis Antonoglou, and David Silver.
\newblock Prioritized experience replay.
\newblock In {\em 4th International Conference on Learning Representations, {ICLR} 2016, San Juan, Puerto Rico, May 2-4, 2016, Conference Track Proceedings}, 2016.

\end{thebibliography}



\end{document}


\maketitle

\appendix
\section*{Appendices}

\section{Extension to LTL and Hybrid-State Models}

\subsection{Ensuring Existence of the Robustness Metric}
\label{appendix:ensuring_existence}

In Section 3 we claim that since any safety property $\varphi$ can be expressed as an LTL formula, we can be sure of the existence of a robustness metric $\varrho^{\varphi}$. To justify this claim, we make use of Signal Temporal Logic (STL) \cite{Maler_Nickovic_2004,Fainekos_Pappas_2009,Madsen_Vaidyanathan_Sadraddini_Vasile_DeLateur_Weiss_Densmore_Belta_2018}. STL extends classical temporal logic by incorporating predicates over real-valued temporal signals $x(t): \mathbb{R} \to \mathbb{R}^{n}$ and, for a given STL formula $\varphi_{\mathrm{STL}}$, returns a real-valued robustness signal $\varrho^{\varphi_{\mathrm{STL}}}(x,t)$, such that $\varrho^{\varphi_{\mathrm{STL}}}(x,t) \geq 0$ when $x(t) \models \varphi_{\mathrm{STL}}$ and $\varrho^{\varphi_{\mathrm{STL}}}(x,t) < 0$  when $x(t) \not\models \varphi_{\mathrm{STL}}$.

To apply STL in the context of RL, we will define $x(t) = x'(s(t))$, where signal mapping $x'(s): \mathcal{S} \to \mathbb{R}^{n}$ maps MDP state $s$ to the real-valued signals upon which STL properties are defined. If our safety property under mapping $x'(s)$ can be expressed as an STL formula (i.e. $\varphi = \varphi_{\mathrm{STL}}$), our robustness metric $\varrho^{\varphi}(\traj{}{})$ is the same as the STL robustness signal $\varrho^{\varphi_{\mathrm{STL}}}(x'(s),t)$ evaluated on finite-length, discrete-time input signal traces (i.e. trajectories) $\traj{}{}$; thus, for properties expressed as an STL formula, we can be sure of the existence of a robustness metric.

To complete the claim, we note that LTL formulae are constructed from a set of atomic propositions $\mathcal{A}\mathcal{P}$, and in the context of RL, each atomic proposition is computed using a function of MDP state known as the labelling function $L(s): \mathcal{S} \to 2^{\mathcal{A}\mathcal{P}}$. By choosing $x'(s)$ such that $L(s) = \mathcal{H}(x'(s))$ where $\mathcal{H}$ is the unit step function, we can build an equivalent STL formula $\varphi_{\mathrm{STL}}$ such that $\traj{}{} \models \varphi \Leftrightarrow \traj{}{} \models \varphi_{\mathrm{STL}}$. In other words, for any LTL formula $\varphi$ that we want our trajectories to satisfy, there exists an equivalent STL formula $\varphi_{\mathrm{STL}}$ that is satisfied by all trajectories that satisfy $\varphi$ (and only those trajectories). Appendix \ref{appendix:LTL_case_studies} provides an example for the time-bounded reach-avoid property.

It is worth noting that, when considering safety only (such as in the main paper), we can assume that $L(s)$ is simply an indicator function for whether safety has been violated at a given MDP state.

\subsection{LTL Specifications and Hybrid-State Models}

SPoRt extends beyond safety properties to encompass any property $\varphi$ expressible as an LTL formula. Moreover, our theoretical results generalize to hybrid-state models with discrete state components, making them applicable to product MDPs in reinforcement learning problems with general LTL specifications \cite{LCRL}. In this section we redefine our system setup to accommodate these generalizations.

We consider a model-free episodic RL setup where an agent interacts with an unknown environment modeled as a Markov Decision Process (MDP) \cite{Sutton_Barto_2014}, specified by the tuple $\langle(\mathcal{S},\mathcal{Q}), \mathcal{A}, p, \mu, r_\mathrm{task}, \mathcal{A}\mathcal{P}, L \rangle$, with a hybrid state space ($\mathcal{S}, \mathcal{Q}$), where $\mathcal{S}$ is the continuous component and $\mathcal{Q}$ is the discrete component, and continuous action space $\mathcal{A}$. $p(s',q'|a,s,q): \mathcal{S} \times \mathcal{Q} \times \mathcal{A} \to \Delta(\mathcal{S} \times \mathcal{Q})$ and $\mu(s,q) \in \Delta(\mathcal{S} \times \mathcal{Q})$ are the (unknown) state-transition and initial state distributions, respectively. $r_\mathrm{task}(s,q,a): \mathcal{S} \times \mathcal{Q} \times \mathcal{A} \to \mathbb{R}$ is the task-specific reward.
$\mathcal{A}\mathcal{P}$ is a finite set of atomic propositions and the labeling function $L(s,q): \mathcal{S} \times \mathcal{Q} \to 2^{\mathcal{A}\mathcal{P}}$ assigns to each state $(s,q) \in (\mathcal{S},\mathcal{Q})$ a set of atomic propositions $L(s,q) \subseteq 2^{\mathcal{A}\mathcal{P}}$~\cite{LCRL}.
We will consider learning a stochastic policy $\policy(a|s,q): \mathcal{S} \times \mathcal{Q} \to \Delta(\mathcal{A})$ in a model-free setup. Let trajectory $\traj{p,\policy}{\st,T} = (\st, \s{t+1}, \ldots, \s{t+T})^{p,\policy}$ denote a realisation of the closed-loop system with state transition distribution $p$, starting at state $\st = (s_{t}, q_{t})$ and evolving for $T$ time steps, using policy $\policy$. 

We define our safety property as a Linear Temporal Logic (LTL) formula $\varphi$ \cite{Pnueli_1977}, using atomic propositions $\mathcal{A}\mathcal{P}$. We use the LTL semantics in \cite{LCRL} to define the satisfaction relation $\traj{}{} \models \varphi$, considering $\traj{}{}$ to be a path of finite length in the MDP. We require $\traj{p,\policy}{\sz,T} \models \varphi$ with high probability, while also attempting to accumulate as much task-specific reward $r_\mathrm{task}$ as possible.

\section{Proofs of Theorems}

Note that all proofs of theorems are written for the general case of hybrid-state models wherein the MDP state space becomes $(\mathcal{S},\mathcal{Q})$; $s \in \mathcal{S}$ is the continuous state component and $q \in \mathcal{Q}$ is the discrete state component. To apply the proofs to the special case of continuous-state MDPs, simply let $|\mathcal{Q}| = 1$.

For the hybrid-state model, we characterize the probability that a sampled trajectory $\traj{p,\policy}{\sz \sim \mu,T} = (\sz \sim \mu, \s{1}, \ldots, \sT)^{p,\policy}$ is such that $\sz \in (\mathcal{S}_{0},\mathcal{Q}_{0}), \dots, \sT \in (\mathcal{S}_{T},\mathcal{Q}_{T})$ as a forward recursion, based on work in~\cite{additive_difference,additive_difference_2}. Let $\mathds{1}_{\mathcal{S}_{t},\mathcal{Q}_{t}}(s,q):\mathcal{S} \times \mathcal{Q} \to \{0, 1\}$ be the indicator function for $s \in \mathcal{S}_{t}, q \in \mathcal{Q}_{t}$, and define functions $W_{t}^{\mu,p,\policy}(s,q): \mathcal{S} \times \mathcal{Q} \to \mathbb{R}_{+}$, characterized as
\begin{gather}
    W_{t+1}^{\mu,p,\policy}(s',q') = \mathds{1}_{\mathcal{S}_{t+1},\mathcal{Q}_{t+1}}(s',q') \sum_{q \in \mathcal{Q}} \int_{\mathcal{S}} P^{p,\policy}(s',q' | s,q) W_{t}^{\mu,p,\policy}(s,q)ds\\
    \text{and}
    \quad
    W_{0}^{\mu,p,\policy}(s',q') = \mathds{1}_{\mathcal{S}_{0},\mathcal{Q}_{0}}(s',q')\mu(s',q'), \quad
    \text{where} \quad P^{p,\policy}(s',q' | s,q) = \int_{\mathcal{A}} p(s',q' | a,s,q) \policy(a | s, q) da.
\end{gather}
It holds that $\Prb\{\traj{p,\policy}{\sz \sim \mu,T}: \sz \in (\mathcal{S}_{0},\mathcal{Q}_{0}), \dots, \sT \in (\mathcal{S}_{T},\mathcal{Q}_{T})\} = \sum_{q \in \mathcal{Q}} \int_{\mathcal{S}} W_{T}^{\mu,p,\policy}(s,q) ds$.

\begin{theorem}
\label{thm:scenario}
    If $\varrho^{\varphi}((\traj{p,\policy}{\sz \sim \mu,T})_{i}) \geq 0$ for all $N$ scenarios $(\traj{p,\policy}{\sz \sim \mu,T})_{i}$ in $\mathcal{D}_{\mu}^{p,\policy}$ = $\{(\traj{p,\policy}{\sz \sim \mu,T})_{i}\}_{i=1}^{N}$, then with confidence $1-\beta$, where $\beta = \left(1 - \eps\right)^{N}$, the probability of drawing a new scenario $\traj{p,\policy}{\sz \sim \mu,T}$ such that $\varrho^{\varphi}(\traj{p,\policy}{\sz \sim \mu,T}) < 0$ is at most $\eps$.
\end{theorem}
\begin{proof}
We will use the theoretical foundations for the scenario approach outlined by \textit{Campi et. al.} \cite{scenario_approach}. Suppose that $\Delta$ represents the set of all possible scenarios, with unknown scenario probability distribution. Given N independent random samples $\delta_{i} \in \Delta$, consider the classical scenario program \ref{eqn:SPN}:
\begin{equation}
    \label{eqn:SPN}
    \tag{$SP_N$}
    \begin{aligned}
        &\min_{r \in \mathcal{R}}~ & &c^T r \\
        &\mathrm{subject~to}~ & & r \in \bigcap_{i=1}^N \mathcal{R}_{\delta_i},
    \end{aligned}
\end{equation}
where $\mathcal{R}, \mathcal{R}_{\delta} \subseteq \R^{d}$ and $d$ is the number of decision variables in the scenario program.

\begin{definition}[violation set and violation probability]
    The violation set of a given $r \in \mathcal{R}_{\delta}$ is the set $\{\delta \in \Delta: r \notin \mathcal{R}_{\delta}\}$. The violation probability (or just violation) of a given $r \in \mathcal{R}_{\delta}$ is the probability of the violation set of $\mathcal{R}_{\delta}$, that is $V(r) \coloneqq \Prb\{\delta \in \Delta: r \notin \mathcal{R}_{\delta}\}$.
\end{definition}

\begin{assumption}[convexity]
\label{asmp:convexity}
    $\mathcal{R}$ and $\mathcal{R}_{\delta}$, $\delta \in \Delta$, are convex closed sets.
\end{assumption}
\begin{assumption}[existence and uniqueness]
\label{asmp:existsunique}
    For every $N$ and for every sample set $\{\delta_1, \ldots, \delta_N\}$, the solution of the program (3.1) exists and is unique.
\end{assumption}

Letting $r^*$ be the solution of \ref{eqn:SPN} we wish to consider the quantification of $V(r^*)$ and present the following standard result:
\begin{lemma}
\label{lem:PAC}
    Let $N \geq d$. Under Assumptions \ref{asmp:convexity} and \ref{asmp:existsunique}, it holds that 
    \begin{equation}
        \Prb^N\{V(r^*) > \epsilon\} \leq \sum_{i=0}^{d-1} \binom{N}{i} \eps^i (1-\eps)^{N-i} \leq \beta.
    \end{equation}
\end{lemma}
Since $d=1$ for our problem, the upper bound in Lemma \ref{lem:PAC} simplifies to
\begin{equation}
    \Prb^N\{V(r^*) > \epsilon\} \leq (1-\eps)^{N} \leq \beta.
\end{equation}

Now, note that the problem of checking $\varrho^{\varphi}((\traj{p,\policy}{\sz \sim \mu,T})_{i}) \geq r^{*}_{\mu,p,\policy}$ for all $N$ scenarios in $\mathcal{D}_{\mu}^{p,\policy}$ = $\{(\traj{p,\policy}{\sz \sim \mu,T})_{i}\}_{i=1}^{N}$ can be written as solving for $r^{*}_{\mu,p,\policy}$ by means of the convex program
\begin{equation}
    \label{eqn:rstar}
    \begin{aligned}
        &\min_{r_{\mu,p,\policy} \in [-a,b]}~ & &-r_{\mu,p,\policy}\\
        &\mathrm{subject~to}~ & & r_{\mu,p,\policy} \in \bigcap_{i=1}^N [-a, \varrho_i].
    \end{aligned}
\end{equation}
where $\varrho_i = \varrho^{\varphi}((\traj{p,\policy}{\sz \sim \mu,T})_{i})$. The set $[-a, \varrho_i]$ is both closed and convex, and $r_{\mu,p,\policy} = -a$ is always a possible solution, thus we consider assumption \ref{asmp:convexity} and \ref{asmp:existsunique} to hold.

We see that this problem is equivalent to \ref{eqn:SPN} and we can apply Lemma \ref{lem:PAC}; with confidence $1-\beta$, where $\beta = \left(1 - \eps\right)^{N}$, it holds that $\Prb\{\varrho^{\varphi}(\traj{p,\policy}{\sz \sim \mu,T}) \in \Delta_{\mu}^{p,\policy}: \varrho^{\varphi}(\traj{p,\policy}{\sz \sim \mu,T}) < r^{*}_{\mu,p,\policy}\} = \Prb\{\varrho^{\varphi}(\traj{p,\policy}{\sz \sim \mu,T}) \in \Delta_{\mu}^{p,\policy}: r^{*}_{\mu,p,\policy} \notin [-a, \varrho^{\varphi}(\traj{p,\policy}{\sz \sim \mu,T})]\} \leq \eps$, where $\Delta_{\mu}^{p,\policy}$ is the (unknown) distribution of trajectories in the task environment under $p$ and $\policy$ starting from $\sz \sim \mu$. To complete the proof, simply set $r^{*}_{\mu,p,\policy} = 0$.
\end{proof}

\begin{theorem}
\label{thm:bound_traj_prob}
    Suppose that, for a set of coefficients $\alpha_{t} \in \mathbb{R}_{+}$, we could constrain $\mu_{2}$, $p_{2}$ and $\policy_{2}$ so as to enforce the following bounds for all $t = 1, \dots, T$:
    \begin{gather}
        \sum_{q \in \mathcal{Q}} \int_{\mathcal{S}} P^{p_{2},\policy_{2}}(s',q' | s,q) W_{t-1}^{\mu_{1},p_{1},\policy_{1}}(s,q) ds
        \leq \alpha_{t} \sum_{q \in \mathcal{Q}} \int_{\mathcal{S}} P^{p_{1},\policy_{1}}(s',q' | s,q) W_{t-1}^{\mu_{1},p_{1},\policy_{1}}(s,q) ds \\
        \text{and} \quad \mu_{2}(s',q') \leq \alpha_{0} \mu_{1}(s',q')
        \quad \forall s' \in \mathcal{S}, q' \in \mathcal{Q}.
    \end{gather}
    It thus holds that
    \begin{equation}
        \Prb\{\traj{p_{2},\policy_{2}}{\sz \sim \mu_{2},T}: \sz \in (\mathcal{S}_{0},\mathcal{Q}_{0}), \dots, \sT \in (\mathcal{S}_{T},\mathcal{Q}_{T})\}
        \leq \Prb\{\traj{p_{1},\policy_{1}}{\sz \sim \mu_{1},T}: \sz \in (\mathcal{S}_{0},\mathcal{Q}_{0}), \dots, \sT \in (\mathcal{S}_{T},\mathcal{Q}_{T})\} \prod_{t=0}^{T}\alpha_{t}.
    \end{equation}
\end{theorem}
\begin{proof}
    \begin{align}
        W_{1}^{\mu_{2},p_{2},\policy_{2}}(s',q') & = \mathds{1}_{\mathcal{S}_{1},\mathcal{Q}_{1}}(s',q') \sum_{q \in \mathcal{Q}} \int_{\mathcal{S}} P^{p_{2},\policy_{2}}(s',q' | s,q) W_{0}^{\mu_{2},p_{2},\policy_{2}}(s,q) ds \\
        & \leq \alpha_{0} \mathds{1}_{\mathcal{S}_{1},\mathcal{Q}_{1}}(s',q') \sum_{q \in \mathcal{Q}} \int_{\mathcal{S}} P^{p_{2},\policy_{2}}(s',q' | s,q) W_{0}^{\mu_{1},p_{1},\policy_{1}}(s,q) ds \\
        & \leq \alpha_{0} \alpha_{1} \mathds{1}_{\mathcal{S}_{1},\mathcal{Q}_{1}}(s',q') \sum_{q \in \mathcal{Q}} \int_{\mathcal{S}} P^{p_{1},\policy_{1}}(s',q' | s,q) W_{0}^{\mu_{1},p_{1},\policy_{1}}(s,q) ds \\
        & = W_{1}^{\mu_{1},p_{1},\policy_{1}}(s',q')\prod_{t'=0}^{1}\alpha_{t'}.
    \end{align}
    
    Assuming that $W_{t}^{\mu_{2},p_{2},\policy_{2}}(s',q') \leq W_{t}^{\mu_{1},p_{1},\policy_{1}}(s',q')\prod_{t'=0}^{t}\alpha_{t'}$ and $1 \leq t \leq T-1$,
    \begin{align}
        W_{t+1}^{\mu_{2},p_{2},\policy_{2}}(s',q') & = \mathds{1}_{\mathcal{S}_{t+1},\mathcal{Q}_{t+1}}(s',q') \sum_{q \in \mathcal{Q}} \int_{\mathcal{S}} P^{p_{2},\policy_{2}}(s',q' | s,q) W_{t}^{\mu_{2},p_{2},\policy_{2}}(s,q) ds \\
        & \leq \mathds{1}_{\mathcal{S}_{t+1},\mathcal{Q}_{t+1}}(s',q') \sum_{q \in \mathcal{Q}} \int_{\mathcal{S}} P^{p_{2},\policy_{2}}(s',q' | s,q) W_{t}^{\mu_{1},p_{1},\policy_{1}}(s,q) ds \prod_{t'=0}^{t}\alpha_{t'} \\
        & \leq \alpha_{t+1} \mathds{1}_{\mathcal{S}_{t+1},\mathcal{Q}_{t+1}}(s',q') \sum_{q \in \mathcal{Q}} \int_{\mathcal{S}} P^{p_{1},\policy_{1}}(s',q' | s,q) W_{t}^{\mu_{1},p_{1},\policy_{1}}(s,q) ds \prod_{t'=0}^{t}\alpha_{t'} \\
        & = W_{t+1}^{\mu_{1},p_{1},\policy_{1}}(s',q')\prod_{t'=0}^{t+1}\alpha_{t'}.
    \end{align}
    
    From these two results we have proved by induction that $W_{t}^{\mu_{2},p_{2},\policy_{2}}(s') \leq W_{t}^{\mu_{1},p_{1},\policy_{1}}(s')\prod_{t'=1}^{t}\alpha_{t'}$ for all $t = 1, \dots, T$. Marginalizing over all $s' \in \mathcal{S}, q' \in \mathcal{Q}$ for $t=T$, we arrive at the following bound:
    \begin{equation}
        \Prb\{\traj{p_{2},\policy_{2}}{\sz \sim \mu_{2},T}: \sz \in (\mathcal{S}_{0},\mathcal{Q}_{0}), \dots, \sT \in (\mathcal{S}_{T},\mathcal{Q}_{T})\}
        \leq \Prb\{\traj{p_{1},\policy_{1}}{\sz \sim \mu_{1},T}: \sz \in (\mathcal{S}_{0},\mathcal{Q}_{0}), \dots, \sT \in (\mathcal{S}_{T},\mathcal{Q}_{T})\} \prod_{t=1}^{T}\alpha_{t}.
    \end{equation}
\end{proof}

\begin{theorem}
\label{thm:policy_ratio}
    Suppose the following constraint holds:
    \begin{equation}
        p_{2}(s',q'|a,s,q) \policy_{2}(a|s,q)
        \leq \alpha_{t} p_{1}(s',q'|a,s,q) \policy_{1}(a|s,q) \quad \forall a \in \mathcal{A}, s \in \mathcal{S}, q \in \mathcal{Q}.
    \end{equation}
    It thus holds that
    \begin{equation}
        \sum_{q \in \mathcal{Q}} \int_{\mathcal{S}} P^{p_{2},\policy_{2}}(s',q' | s,q) W_{t-1}^{\mu_{1},p_{1},\policy_{1}}(s,q) ds
        \leq \alpha_{t} \sum_{q \in \mathcal{Q}} \int_{\mathcal{S}} P^{p_{1},\policy_{1}}(s',q' | s,q) W_{t-1}^{\mu_{1},p_{1},\policy_{1}}(s,q) ds
        \quad \forall s'\in \mathcal{S}, q' \in \mathcal{Q}.
    \end{equation}
\end{theorem}
\begin{proof}
    \begin{align}
        p_{2}(s',q'|a,s,q) \policy_{2}(a|s,q) & \leq \alpha_{t} p_{1}(s',q'|a,s,q) \policy_{1}(a|s,q) \quad \forall a \in \mathcal{A}, s \in \mathcal{S}, q \in \mathcal{Q} \\
        \int_{\mathcal{A}} p_{2}(s',q'|a,s,q) \policy_{2}(a|s,q) da & \leq \alpha_{t} \int_{\mathcal{A}} p_{1}(s',q'|a,s,q) \policy_{1}(a|s,q) da \quad \forall s', s \in \mathcal{S}, q \in \mathcal{Q}
        \\
        P^{p_{2},\policy_{2}}(s',q' | s,q) & \leq \alpha_{t} P^{p_{1},\policy_{1}}(s',q' | s,q) \quad \forall s',s \in \mathcal{S}, q',q \in \mathcal{Q}
        \\
        \sum_{q \in \mathcal{Q}} \int_{\mathcal{S}} P^{p_{2},\policy_{2}}(s',q' | s,q) W_{t-1}^{\mu_{1},p_{1},\policy_{1}}(s,q) ds & \leq \alpha_{t} \sum_{q \in \mathcal{Q}} \int_{\mathcal{S}} P^{p_{1},\policy_{1}}(s',q' | s,q) W_{t-1}^{\mu_{1},p_{1},\policy_{1}}(s,q) ds \quad \forall s'\in \mathcal{S}, q' \in \mathcal{Q}.
    \end{align}
\end{proof}

\begin{theorem}
\label{thm:task_fail_prob}
    Suppose that
    \begin{equation}
        \Prb\{\varrho^{\varphi}(\traj{p_{1},\policy_{1}}{\sz \sim \mu_{1},T}) \in \Delta: \varrho^{\varphi}(\traj{p_{1},\policy_{1}}{\sz \sim \mu_{1},T}) < 0\} \leq \eps,
    \end{equation}
    and for all $t = 1, \dots, T$,
    \begin{gather}
        p_{2}(s',q'|a,s,q) \policy_{2}(a|s,q)
        \leq \alpha_{t} p_{1}(s',q'|a,s,q) \policy_{1}(a|s,q) \; \text{and} \;
        \mu_{2}(s',q') \leq \alpha_{0} \mu_{1}(s',q') \quad \forall a \in \mathcal{A}, s \in \mathcal{S}, q \in \mathcal{Q}.
    \end{gather}
    It thus holds that
    \begin{equation}
        \Prb\{\varrho^{\varphi}(\traj{p_{2},\policy_{2}}{\sz \sim \mu_{2},T}) \in \Delta: \varrho^{\varphi}(\traj{p_{2},\policy_{2}}{\sz \sim \mu_{2},T}) < 0\} \leq \eps \prod_{t=0}^{T}\alpha_{t}.
    \end{equation}
\end{theorem}
\begin{proof}
    We will define the notation
    \begin{equation}
        \sum_{\{(\mathcal{S}_{0},\mathcal{Q}_{0}),\dots,(\mathcal{S}_{T},\mathcal{Q}_{T}) : \traj{}{}\not\models\varphi\}} \Prb\{\traj{p,\policy}{\sz \sim \mu,T}: \sz \in (\mathcal{S}_{0},\mathcal{Q}_{0}), \dots, \sT \in (\mathcal{S}_{T},\mathcal{Q}_{T})\}
    \end{equation}
    as the sum of probabilities $\Prb\{\traj{p,\policy}{\sz \sim \mu,T}: \sz \in (\mathcal{S}_{0},\mathcal{Q}_{0}), \dots, \sT \in (\mathcal{S}_{T},\mathcal{Q}_{T})\}$ over all permutations of possible sets $(\mathcal{S}_{0},\mathcal{Q}_{0}),\dots,(\mathcal{S}_{T},\mathcal{Q}_{T}) \subseteq (\mathcal{S},\mathcal{Q})$ such that the event of drawing trajectory $\traj{p,\policy}{\sz \sim \mu,T}$ where $\sz \in (\mathcal{S}_{0},\mathcal{Q}_{0}), \dots, \sT \in (\mathcal{S}_{T},\mathcal{Q}_{T})$ results in $\traj{p,\policy}{\sz \sim \mu,T}\not\models\varphi$, and (importantly) also such that each event in the sum is mutually exclusive to all other events in the sum. In other words, the notation represents the sum of the probabilities of all (mutually exclusive) events of $(\mathcal{S}_{0},\mathcal{Q}_{0}),\dots,(\mathcal{S}_{T},\mathcal{Q}_{T}) \subseteq (\mathcal{S},\mathcal{Q})$ being such that $\traj{p,\policy}{\sz \sim \mu,T}\not\models\varphi$, and the sum is thus equal to the probability that $\traj{p,\policy}{\sz \sim \mu,T}\not\models\varphi$ and therefore $\varrho^{\varphi}(\traj{p,\policy}{\sz \sim \mu,T}) < 0$.
    
    To illustrate, consider the case for safety, with unsafe set $\mathcal{S}_{U} \subset \mathcal{S}$ (let $\mathcal{Q}_{U} = \emptyset$). For this case, the permutations are such that $\mathcal{S}_{t} \in \{\mathcal{S}_{U}, \mathcal{S}_{U}^{'}\}$ for all $t=0,\dots,T$ and there exists $t$ such that $\mathcal{S}_{t} = \mathcal{S}_{U}$, so the sum is over the probabilities $\Prb\{\traj{p,\policy}{\sz \sim \mu,T}: \sz \in (\mathcal{S}_{0},\mathcal{Q}_{0}), \dots, \sT \in (\mathcal{S}_{T},\mathcal{Q}_{T})\}$ for all (mutually exclusive) events of drawing $\traj{p,\policy}{\sz \sim \mu,T}$ such that least one set $\mathcal{S}_{t}$ is the unsafe set $\mathcal{S}_{U}$ (and all other sets are either $\mathcal{S}_{U}$ or $\mathcal{S}_{U}^{'}$).
    
    Now consider the probability that sampled trajectory $\traj{p_{2},\policy_{2}}{\sz \sim \mu_{2},T}$ is such that $\varrho^{\varphi}(\traj{p_{2},\policy_{2}}{\sz \sim \mu_{2},T}) < 0$:
    \begin{align}
        &\Prb\{\varrho^{\varphi}(\traj{p_{2},\policy_{2}}{\sz \sim \mu_{2},T}) \in \Delta: \varrho^{\varphi}(\traj{p_{2},\policy_{2}}{\sz \sim \mu_{2},T}) < 0\} \\
        = & \sum_{\{(\mathcal{S}_{0},\mathcal{Q}_{0}),\dots,(\mathcal{S}_{T},\mathcal{Q}_{T}) : \traj{}{}\not\models\varphi\}} \Prb\{\traj{p_{2},\policy_{2}}{\sz \sim \mu_{2},T}: \sz \in (\mathcal{S}_{0},\mathcal{Q}_{0}), \dots, \sT \in (\mathcal{S}_{T},\mathcal{Q}_{T})\} \\
        \leq & \sum_{\{(\mathcal{S}_{0},\mathcal{Q}_{0}),\dots,(\mathcal{S}_{T},\mathcal{Q}_{T}) : \traj{}{}\not\models\varphi\}} \Prb\{\traj{p_{1},\policy_{1}}{\sz \sim \mu_{1},T}: \sz \in (\mathcal{S}_{0},\mathcal{Q}_{0}), \dots, \sT \in (\mathcal{S}_{T},\mathcal{Q}_{T})\} \prod_{t=0}^{T}\alpha_{t} \\
        = & \Prb\{\varrho^{\varphi}(\traj{p_{1},\policy_{1}}{\sz \sim \mu_{1},T}) \in \Delta: \varrho^{\varphi}(\traj{p_{1},\policy_{1}}{\sz \sim \mu_{1},T}) < 0\} \prod_{t=0}^{T}\alpha_{t} \\
        \leq & \eps \prod_{t=0}^{T}\alpha_{t},
    \end{align}
    where Theorems \ref{thm:bound_traj_prob} and \ref{thm:policy_ratio} are used to obtain the first inequality.
\end{proof}

\begin{theorem}
\label{thm:policy_projection}
    Assuming diagonal Gaussian policies and using KL divergence as the distance metric for projection, the means and standard deviations of projected policy $\policyproj$ can be computed from $\policybase$ and $\policytask$ by solving the following convex optimization problem at each time step:
    \begin{equation}
        \begin{aligned}
            &\min_{\boldsymbol{\mu}_\mathrm{proj},\boldsymbol{\sigma}_\mathrm{proj}}~ & &J(\boldsymbol{\mu}_\mathrm{proj},\boldsymbol{\sigma}_\mathrm{proj}) \\
            &\mathrm{subject~to}~ & & \prod_{i=1}^{n}\left(\frac{\sigma_{\mathrm{base},i}}{\sigma_{\mathrm{proj},i}} e^{\frac{1}{2} \frac{\left(\mu_{\mathrm{proj},i} - \mu_{\mathrm{base},i}\right)^{2}}{\sigma_{\mathrm{base},i}^{2} - \sigma_{\mathrm{proj},i}^{2}} }\right) \leq \alpha,\\
            & ~ & & 0< \sigma_{\mathrm{proj},i} < \sigma_{\mathrm{base},i} \quad \forall i=1,\dots,n,
        \end{aligned}
    \end{equation}
    \begin{equation}
        \text{where} \quad J(\boldsymbol{\mu}_\mathrm{proj},\boldsymbol{\sigma}_\mathrm{proj})
        = \sum_{i=1}^{n} \left( - 2\ln\left(\sigma_{\mathrm{proj},i}\right) + \frac{\sigma_{\mathrm{proj},i}^{2}}{\sigma_{\mathrm{task},i}^{2}} + \frac{\left(\mu_{\mathrm{proj},i} - \mu_{\mathrm{task},i}\right)^{2}}{\sigma_{\mathrm{task},i}^{2}} \right).
    \end{equation}
\end{theorem}
\begin{proof}
    The diagonal Gaussian policy distribution for an $n$-dimensional continuous action space is given by
    \begin{equation}
        \policy(\mathbf{a}|s,q) = \mathcal{N}\left(\mathbf{a};\boldsymbol{\mu},\mathbf{\Sigma}\right) = \frac{1}{(2\pi)^{\frac{n}{2}}\lvert\mathbf{\Sigma}\rvert^\frac{1}{2}} e^{-\frac{1}{2}\left(\mathbf{a} - \boldsymbol{\mu}\right)^{T}\mathbf{\Sigma}^{-1}\left(\mathbf{a} - \boldsymbol{\mu}\right)}, \quad \mathbf{\Sigma} = \mathrm{diag}\left(\boldsymbol{\sigma}^{2}\right).
    \end{equation}
    
    We will first derive an expression for $\alpha$ as a function of the projected policy means $\boldsymbol{\mu}_{\mathrm{proj}}$ and standard deviations $\boldsymbol{\sigma}_{\mathrm{proj}}$. The policy ratio can be written as
    \begin{align}
        \frac{\policyproj(\mathbf{a}|s,q)}{\policybase(\mathbf{a}|s,q)} & = \frac{\frac{1}{(2\pi)^{\frac{n}{2}}\lvert\mathbf{\Sigma}_{\mathrm{proj}}\rvert^\frac{1}{2}} e^{-\frac{1}{2}\left(\mathbf{a} - \boldsymbol{\mu}_{\mathrm{proj}}\right)^{T}\mathbf{\Sigma}_{\mathrm{proj}}^{-1}\left(\mathbf{a} - \boldsymbol{\mu}_{\mathrm{proj}}\right)}}{\frac{1}{(2\pi)^{\frac{n}{2}}\lvert\mathbf{\Sigma}_{\mathrm{base}}\rvert^\frac{1}{2}} e^{-\frac{1}{2}\left(\mathbf{a} - \boldsymbol{\mu}_{\mathrm{base}}\right)^{T}\mathbf{\Sigma}_{\mathrm{base}}^{-1}\left(\mathbf{a} - \boldsymbol{\mu}_{\mathrm{base}}\right)}}\\
        & = \frac{\lvert\mathbf{\Sigma}_{\mathrm{base}}\rvert^\frac{1}{2}}{\lvert\mathbf{\Sigma}_{\mathrm{proj}}\rvert^\frac{1}{2}} e^{-\frac{1}{2}\left(\mathbf{a} - \boldsymbol{\mu}_{\mathrm{proj}}\right)^{T}\mathbf{\Sigma}_{\mathrm{proj}}^{-1}\left(\mathbf{a} - \boldsymbol{\mu}_{\mathrm{proj}}\right)+\frac{1}{2}\left(\mathbf{a} - \boldsymbol{\mu}_{\mathrm{base}}\right)^{T}\mathbf{\Sigma}_{\mathrm{base}}^{-1}\left(\mathbf{a} - \boldsymbol{\mu}_{\mathrm{base}}\right)}.
    \end{align}
    
    Note that $\mathbf{\Sigma} = \mathrm{diag}\left(\boldsymbol{\sigma}^{2}\right)$, thus $\lvert\mathbf{\Sigma}\rvert = \prod_{i=1}^{n}\sigma_{i}^{2}$ and $\mathbf{\Sigma}^{-1} = \mathrm{diag}\left(\frac{1}{\boldsymbol{\sigma}^{2}}\right)$. Substituting,
    \begin{equation}
        \frac{\policyproj(\mathbf{a}|s,q)}{\policybase(\mathbf{a}|s,q)} = \frac{\prod_{i=1}^{n}\sigma_{\mathrm{base},i}}{\prod_{i=1}^{n}\sigma_{\mathrm{proj},i}} e^{-\frac{1}{2}\sum_{i=1}^{n}\left( \frac{1}{\sigma_{\mathrm{proj},i}^{2}} \left( a_{i} - \mu_{\mathrm{proj},i} \right)^{2} - \frac{1}{\sigma_{\mathrm{base},i}^{2}} \left( a_{i} -  \mu_{\mathrm{base},i}  \right)^{2} \right)}.
    \end{equation}
    
    We want to find the location of the stationary point $\mathbf{a}^{*}$. The stationary point $a_{i}^{*}$ is such that
    \begin{align}
        \frac{2}{\sigma_{\mathrm{proj},i}^{2}} \left( a_{i}^{*} - \mu_{\mathrm{proj},i} \right) - \frac{2}{\sigma_{\mathrm{base},i}^{2}} \left( a_{i}^{*} -  \mu_{\mathrm{base},i}  \right) & = 0\\
        \frac{1}{\sigma_{\mathrm{proj},i}^{2}} \left( a_{i}^{*} - \mu_{\mathrm{proj},i} \right) & = \frac{1}{\sigma_{\mathrm{base},i}^{2}} \left( a_{i}^{*} -  \mu_{\mathrm{base},i}  \right)\\
        \sigma_{\mathrm{base},i}^{2}a_{i}^{*} - \sigma_{\mathrm{base},i}^{2}\mu_{\mathrm{proj},i} & = \sigma_{\mathrm{proj},i}^{2}a_{i}^{*} - \sigma_{\mathrm{proj},i}^{2}\mu_{\mathrm{base},i}\\
        \left(\sigma_{\mathrm{base},i}^{2} - \sigma_{\mathrm{proj},i}^{2}\right)a_{i}^{*} & = \sigma_{\mathrm{base},i}^{2}\mu_{\mathrm{proj},i} - \sigma_{\mathrm{proj},i}^{2}\mu_{\mathrm{base},i}\\
        \therefore a_{i}^{*} & = \frac{\sigma_{\mathrm{base},i}^{2}}{\sigma_{\mathrm{base},i}^{2} - \sigma_{\mathrm{proj},i}^{2}}\mu_{\mathrm{proj},i} - \frac{\sigma_{i}^{2}}{\sigma_{\mathrm{base},i}^{2} - \sigma_{\mathrm{proj},i}^{2}}\mu_{\mathrm{base},i}.
    \end{align}
    
    We can use this to find an expression for $\left( a_{i}^{*} - \mu_{\mathrm{proj},i} \right)^{2}$:
    \begin{align}
        a_{i}^{*} - \mu_{\mathrm{proj},i} & = \frac{\sigma_{\mathrm{base},i}^{2}}{\sigma_{\mathrm{base},i}^{2} - \sigma_{\mathrm{proj},i}^{2}}\mu_{\mathrm{proj},i}- \frac{\sigma_{\mathrm{base},i}^{2} - \sigma_{\mathrm{proj},i}^{2}}{\sigma_{\mathrm{base},i}^{2} - \sigma_{\mathrm{proj},i}^{2}}\mu_{\mathrm{proj},i} - \frac{\sigma_{\mathrm{proj},i}^{2}}{\sigma_{\mathrm{base},i}^{2} - \sigma_{\mathrm{proj},i}^{2}}\mu_{\mathrm{base},i}\\
        & = \frac{\sigma_{\mathrm{proj},i}^{2}}{\sigma_{\mathrm{base},i}^{2} - \sigma_{\mathrm{proj},i}^{2}}\left(\mu_{\mathrm{proj},i} - \mu_{\mathrm{base},i}\right)\\
        \therefore \left( a_{i}^{*} - \mu_{\mathrm{proj},i} \right)^{2} & = \frac{\left(\sigma_{\mathrm{proj},i}^{2}\right)^{2}}{\left(\sigma_{\mathrm{base},i}^{2} - \sigma_{\mathrm{proj},i}^{2}\right)^{2}}\left(\mu_{\mathrm{proj},i} - \mu_{\mathrm{base},i}\right)^{2}.
    \end{align}
    
    We can do the same for $\left( a_{i}^{*} - \mu_{\mathrm{base},i} \right)^{2}$:
    \begin{align}
        a_{i}^{*} - \mu_{\mathrm{base},i} & = \frac{\sigma_{\mathrm{base},i}^{2}}{\sigma_{\mathrm{base},i}^{2} - \sigma_{\mathrm{proj},i}^{2}}\mu_{\mathrm{proj},i} - \frac{\sigma_{\mathrm{proj},i}^{2}}{\sigma_{\mathrm{base},i}^{2} - \sigma_{\mathrm{proj},i}^{2}}\mu_{\mathrm{base},i} - \frac{\sigma_{\mathrm{base},i}^{2} - \sigma_{\mathrm{proj},i}^{2}}{\sigma_{\mathrm{base},i}^{2} - \sigma_{\mathrm{proj},i}^{2}}\mu_{\mathrm{base},i}\\
        & = \frac{\sigma_{\mathrm{base},i}^{2}}{\sigma_{\mathrm{base},i}^{2} - \sigma_{\mathrm{proj},i}^{2}}\left(\mu_{\mathrm{proj},i} - \mu_{\mathrm{base},i}\right)\\
        \therefore \left( a_{i}^{*} - \mu_{\mathrm{base},i} \right)^{2} & = \frac{\left(\sigma_{\mathrm{base},i}^{2}\right)^{2}}{\left(\sigma_{\mathrm{base},i}^{2} - \sigma_{\mathrm{proj},i}^{2}\right)^{2}}\left(\mu_{\mathrm{proj},i} - \mu_{\mathrm{base},i}\right)^{2}.
    \end{align}
    
    With these two results we can write the following:
    \begin{align}
        \frac{1}{\sigma_{\mathrm{proj},i}^{2}}\left( a_{i}^{*} - \mu_{\mathrm{proj},i} \right)^{2} - \frac{1}{\sigma_{\mathrm{base},i}^{2}}\left( a_{i}^{*} - \mu_{\mathrm{base},i} \right)^{2} & = \frac{\sigma_{\mathrm{proj},i}^{2}}{\left(\sigma_{\mathrm{base},i}^{2} - \sigma_{\mathrm{proj},i}^{2}\right)^{2}}\left(\mu_{\mathrm{proj},i} - \mu_{\mathrm{base},i}\right)^{2}\\
        & -\frac{\sigma_{\mathrm{base},i}^{2}}{\left(\sigma_{\mathrm{base},i}^{2} - \sigma_{\mathrm{proj},i}^{2}\right)^{2}}\left(\mu_{\mathrm{proj},i} - \mu_{\mathrm{base},i}\right)^{2}\\
        & = -\frac{\sigma_{\mathrm{base},i}^{2} - \sigma_{\mathrm{proj},i}^{2}}{\left(\sigma_{\mathrm{base},i}^{2} - \sigma_{\mathrm{proj},i}^{2}\right)^{2}}\left(\mu_{\mathrm{proj},i} - \mu_{\mathrm{base},i}\right)^{2}\\
        & = -\frac{\left(\mu_{\mathrm{proj},i} - \mu_{\mathrm{base},i}\right)^{2}}{\sigma_{\mathrm{base},i}^{2} - \sigma_{\mathrm{proj},i}^{2}}.
    \end{align}
    
    If all $\sigma_{\mathrm{proj},i} < \sigma_{\mathrm{base},i}$ then the point $\mathbf{a}^{*}$ is a maximum and yields a finite policy ratio. Substituting this stationary point into our expression for the policy ratio, we arrive at
    \begin{equation}
        \alpha(\boldsymbol{\mu}_{\mathrm{proj}},\boldsymbol{\sigma}_{\mathrm{proj}}) = \max_{\mathbf{a} \in \mathcal{A}}\left(\frac{\policyproj(\mathbf{a} |s,q)}{\policybase(\mathbf{a} |s,q)}\right) = \frac{\prod_{i=1}^{n}\sigma_{\mathrm{base},i}}{\prod_{i=1}^{n}\sigma_{\mathrm{proj},i}} e^{\frac{1}{2}\sum_{i=1}^{n} \frac{\left(\mu_{\mathrm{proj},i} - \mu_{\mathrm{base},i}\right)^{2}}{\sigma_{\mathrm{base},i}^{2} - \sigma_{\mathrm{proj},i}^{2}} }.
    \end{equation}
    We have arrived at an expression for the maximum policy ratio at a given state as a function of the means and standard deviations of the $\policyproj$ at that state. Projecting $\policytask$ onto $\mathit{\Pi}_{\alpha,\policybase}$ at each state becomes a problem of finding $\boldsymbol{\mu}_{\mathrm{proj}}$ and $\boldsymbol{\sigma}_{\mathrm{proj}}$ such that $\alpha(\boldsymbol{\mu}_{\mathrm{proj}},\boldsymbol{\sigma}_{\mathrm{proj}}) \leq \alpha$ while remaining as close as possible to $\policytask$. We can write this as the following constrained minimization problem:
    \begin{equation}
        \begin{aligned}
            &\min_{\boldsymbol{\mu}_\mathrm{proj},\boldsymbol{\sigma}_\mathrm{proj}}~ & &J(\boldsymbol{\mu}_\mathrm{proj},\boldsymbol{\sigma}_\mathrm{proj}) = D_{KL}\left( \mathcal{N}\left(\boldsymbol{\mu}_\mathrm{proj},\mathrm{diag}\left(\boldsymbol{\sigma}_\mathrm{proj}^{2}\right)\right)\lvert\rvert\mathcal{N}\left(\boldsymbol{\mu}_\mathrm{task},\mathrm{diag}\left(\boldsymbol{\sigma}_\mathrm{task}^{2}\right)\right) \right) \\
            &\mathrm{subject~to}~ & & \frac{\prod_{i=1}^{n}\sigma_{\mathrm{base},i}}{\prod_{i=1}^{n}\sigma_{\mathrm{proj},i}} e^{\frac{1}{2}\sum_{i=1}^{n} \frac{\left(\mu_{\mathrm{proj},i} - \mu_{\mathrm{base},i}\right)^{2}}{\sigma_{\mathrm{base},i}^{2} - \sigma_{\mathrm{proj},i}^{2}} } \leq \alpha,\\
            & ~ & & 0< \sigma_{\mathrm{proj},i} < \sigma_{\mathrm{base},i} \quad \forall i=1,\dots,n.
        \end{aligned}
    \end{equation}
    
    We will now write out the KL divergence term that we aim to minimize. It can be shown\footnote{\url{https://statproofbook.github.io/P/mvn-kl.html}} that the KL divergence between two multivariate Gaussians takes the form
    \begin{equation}
        D_{KL}\left( \mathcal{N}\left(\boldsymbol{\mu}_{1},\boldsymbol{\Sigma}_{1}\right)\lvert\rvert\mathcal{N}\left(\boldsymbol{\mu}_{2},\boldsymbol{\Sigma}_{2}\right) \right)
        = \frac{1}{2}\left(\ln\left(\frac{\lvert \boldsymbol{\Sigma}_{2} \rvert}{\lvert  \boldsymbol{\Sigma}_{1} \rvert}\right) - n + \mathrm{tr}\left(\boldsymbol{\Sigma}_{2}^{-1}\boldsymbol{\Sigma}_{1}\right) + \left(\boldsymbol{\mu}_{2} - \boldsymbol{\mu}_{1}\right)^{T}\boldsymbol{\Sigma}_{2}^{-1}\left(\boldsymbol{\mu}_{2} - \boldsymbol{\mu}_{1}\right)\right).
    \end{equation}
    
    Assuming that $\boldsymbol{\Sigma}_{1} = \mathrm{diag}\left(\boldsymbol{\sigma}_{1}^{2}\right)$ and $\boldsymbol{\Sigma}_{2} = \mathrm{diag}\left(\boldsymbol{\sigma}_{2}^{2}\right)$, we can re-write this as
    \begin{align}
            & D_{KL}\left( \mathcal{N}\left(\boldsymbol{\mu}_{1},\mathrm{diag}\left(\boldsymbol{\sigma}_{1}^{2}\right)\right)\lvert\rvert\mathcal{N}\left(\boldsymbol{\mu}_{2},\mathrm{diag}\left(\boldsymbol{\sigma}_{2}^{2}\right)\right) \right) \\
            = & \frac{1}{2}\left(\ln\left(\frac{\prod_{i=1}^{n}\sigma_{2,i}^{2}}{\prod_{i=1}^{n}\sigma_{1,i}^{2}}\right) - n + \sum_{i=1}^{n}\frac{\sigma_{1,i}^{2}}{\sigma_{2,i}^{2}} + \sum_{i=1}^{n}\frac{\left(\mu_{2,i} - \mu_{1,i}\right)^{2}}{\sigma_{2,i}^{2}} \right)\\
            = & \frac{1}{2} \sum_{i=1}^{n} \left( \ln\left(\frac{\sigma_{2,i}^{2}}{\sigma_{1,i}^{2}}\right) - 1 + \frac{\sigma_{1,i}^{2}}{\sigma_{2,i}^{2}} + \frac{\left(\mu_{2,i} - \mu_{1,i}\right)^{2}}{\sigma_{2,i}^{2}} \right)\\
            = & \frac{1}{2} \sum_{i=1}^{n} \left( 2\ln\left(\sigma_{2,i}\right) - 2\ln\left(\sigma_{1,i}\right) - 1 + \frac{\sigma_{1,i}^{2}}{\sigma_{2,i}^{2}} + \frac{\left(\mu_{2,i} - \mu_{1,i}\right)^{2}}{\sigma_{2,i}^{2}} \right).
    \end{align}
    
    Let $\policy_{1} = \policyproj$ and $\policy_{2} = \policytask$. Since we only seek the minimizer of the optimization problem and not the minimum value itself, we can discard constant additive terms in the objective function, leaving us with the following:
    \begin{equation}
        J(\boldsymbol{\mu}_\mathrm{proj},\boldsymbol{\sigma}_\mathrm{proj}) = \sum_{i=1}^{n} \left( - 2\ln\left(\sigma_{\mathrm{proj},i}\right) + \frac{\sigma_{\mathrm{proj},i}^{2}}{\sigma_{\mathrm{task},i}^{2}} + \frac{\left(\mu_{\mathrm{proj},i} - \mu_{\mathrm{task},i}\right)^{2}}{\sigma_{\mathrm{task},i}^{2}} \right).
    \end{equation}
    
    From our analysis, we can rewrite the constrained minimization problem for projection as
    \begin{equation}
        \begin{aligned}
            &\min_{\boldsymbol{\mu}_\mathrm{proj},\boldsymbol{\sigma}_\mathrm{proj}}~ & &\sum_{i=1}^{n} \left( - 2\ln\left(\sigma_{\mathrm{proj},i}\right) + \frac{\sigma_{\mathrm{proj},i}^{2}}{\sigma_{\mathrm{task},i}^{2}} + \frac{\left(\mu_{\mathrm{proj},i} - \mu_{\mathrm{task},i}\right)^{2}}{\sigma_{\mathrm{task},i}^{2}} \right) \\
            &\mathrm{subject~to}~ & & \prod_{i=1}^{n}\left(\frac{\sigma_{\mathrm{base},i}}{\sigma_{\mathrm{proj},i}} e^{\frac{1}{2} \frac{\left(\mu_{\mathrm{proj},i} - \mu_{\mathrm{base},i}\right)^{2}}{\sigma_{\mathrm{base},i}^{2} - \sigma_{\mathrm{proj},i}^{2}} }\right) \leq \alpha,\\
            & ~ & & 0< \sigma_{\mathrm{proj},i} < \sigma_{\mathrm{base},i} \quad \forall i=1,\dots,n.
        \end{aligned}
    \end{equation}
    This problem can be viewed as $n$ constrained minimization problems, one for each dimension in action space, but also jointly coupled by the first constraint. Since not only the objective function and domain but also the feasible set\footnote{While the first constraint is non-convex for general $\sigma_{\mathrm{proj},i}$, it is convex over the domain $\sigma_{\mathrm{proj},i} \in (0, \sigma_{\mathrm{base},i})$, which is enforced by the other constraints.} are convex with respect to each optimization variable, this is a convex minimization problem.
\end{proof}

\section{Observations From the Theoretical Results}

\subsection{On the Conservativeness of the Upper Bound}
\label{appendix:conservativeness_of_upper_bound}

\sloppy The upper bound on the probability of property violation under $(\mu_{2},p_{2},\policy_{2})$, $\Prb\{\varrho^{\varphi}(\traj{p_{2},\policy_{2}}{\sz \sim \mu_{2},T}) \in \Delta: \varrho^{\varphi}(\traj{p_{2},\policy_{2}}{\sz \sim \mu_{2},T}) < 0\} \leq \eps \prod_{t=0}^{T}\alpha_{t}$, obtained by ensuring constraints $\mu_{2}(s',q') \leq \alpha_{0} \mu_{1}(s',q')$ and $p_{2}(s',q'|a,s,q) \policy_{2}(a|s,q) \leq \alpha_{t} p_{1}(s',q'|a,s,q) \policy_{1}(a|s,q)$ for all $a \in \mathcal{A}, s \in \mathcal{S}, q \in \mathcal{Q}$ and applying Theorem \ref{thm:task_fail_prob}, is inherently conservative. This is made clear by observing that $p_{2}(s',q'|a,s,q) \policy_{2}(a|s,q) \leq \alpha_{t} p_{1}(s',q'|a,s,q) \policy_{1}(a|s,q) = p_{2}(s',q'|a,s,q) \policy_{2}(a|s,q) + p_{\mathrm{art}}(s',q'|s,q)$, where $p_{\mathrm{art}}(s',q'|s,q)\geq 0$ is an artificial probability mass that we add to $p_{2} \policy_{2}$ when computing the upper bound to make it equal to $p_{1} \policy_{1}$ but scaled up by $\alpha$. This artificial probability mass increases the upper bound since we are adding to it an `extra' probability that does not actually exist. The bound is only tight when $p_{\mathrm{art}}(s',q'|s,q)= 0$ for all $s,s' \in \mathcal{S}, q,q' \in \mathcal{Q}$, which requires the trivial case where $\alpha_{t} = 1$ where $p_{2} \policy_{2} = p_{1} \policy_{1}$ (and also $\alpha_{0} = 1$ where $\mu_{2}=\mu_{1}$). The larger we make $\alpha_{t}$, the more artificial probability mass we must add, making the bound more conservative. Furthermore, the addition of this artificial probability mass is compounded exponentially when we multiply by $\alpha_{t}$ over $t = 1, \dots, T$, making our total multiplicative increase $\prod_{t=0}^{T}\alpha_{t}$ very conservative for applications with large episode length $T$.

\subsection{Application to Robust Control Problems}
\label{extension_robust_control}

In Section 5 we focused on maintaining a bound on property satisfaction when using a modified policy $\policytask$, however we see that Theorem \ref{thm:task_fail_prob} also provides a bound when $\mu_{2}$ and $p_{2}$ differ from $\mu_{1}$ and $p_{1}$. Thus, our theoretical results can also be applied to robust control settings where we require a bound on property satisfaction for a perturbed system, where nominal system dynamics $(\mu_{\mathrm{nom}},p_{\mathrm{nom}},\policy_{\mathrm{nom}})$, for which an upper bound on probability of property violation is known to be $\eps_{\mathrm{nom}}$ with confidence $1 - \beta$, are perturbed to $(\mu_{\mathrm{per}},p_{\mathrm{per}},\policy_{\mathrm{per}})$. As long as the set of all possible perturbations is known to fall within the constraints outlined by Theorem \ref{thm:task_fail_prob}, an upper bound on probability of property violation for the perturbed system is $\eps_{\mathrm{per}} = \eps_{\mathrm{nom}} \prod_{t=0}^{T}\alpha_{t}$ with confidence $1 - \beta$.

\subsection{On the Relationship Between Stochasticity of the Base Policy and the Size of the Feasible Set of the Projected Policy}
We note that $\mathit{\Pi}_{\alpha, \policybase}$ will generally be larger when $\policybase$ is more exploratory with a higher standard deviation. If $\policybase$ and $\policyproj$ are `tall' Gaussians with small `width', then even small deviations in the mean of $\policyproj$ from that of $\policybase$ will cause a rapid increase in policy ratio, while if they are made `flatter' with larger width, the same deviations in the mean of $\policyproj$ will cause a smaller increase in policy ratio. This means that a more exploratory $\policybase$ requires a lower $\alpha$ and therefore yields lower $\eps_{\mathrm{task}}$ for the same deviation in mean of $\policyproj$ from $\policybase$; this is intuitive, since we would expect safety guarantees from a more exploratory $\policybase$ to be more informative when making guarantees for the modified policy.

\section{Implementing the Policy Projection Method}

The convex minimization problem is solved at every time step using CVXPY \cite{CVXPY_1,CVXPY_2}. In order to use CVXPY, the problem must adhere to the Disciplined Convex Programming (DCP) rules described in its documentation. The DCP-adherent problem given to CVXPY is as follows:
\begin{equation}
    \begin{aligned}
        &\min_{\boldsymbol{\mu}_\mathrm{proj}\in \mathbb{R},\boldsymbol{\sigma}_\mathrm{proj}\in \mathbb{R}_{++}} ~ & & - 2 \sum_{i=1}^{n} \ln\left(\sigma_{\mathrm{proj},i}\right) + \sum_{i=1}^{n} \left(\frac{\sigma_{\mathrm{proj},i}}{\sigma_{\mathrm{task},i}}\right)^{2} + \sum_{i=1}^{n}\left(\frac{\mu_{\mathrm{proj},i} - \mu_{\mathrm{task},i}}{\sigma_{\mathrm{task},i}}\right)^{2} \\
        &\mathrm{subject~to}~ & & \sum_{i=1}^{n}\ln (\sigma_{\mathrm{base},i}) - \sum_{i=1}^{n} \ln (\sigma_{\mathrm{proj},i}) + \frac{1}{2} \sum_{i=1}^{n} \frac{\left(\mu_{\mathrm{proj},i} - \mu_{\mathrm{base},i}\right)^{2}}{\sigma_{\mathrm{base},i}^{2} - \sigma_{\mathrm{proj},i}^{2}} \leq \ln(\alpha),\\
        & ~ & & \boldsymbol{\sigma}_{\mathrm{proj}} + \eta\mathbf{1} \leq \boldsymbol{\sigma}_{\mathrm{base}}.
    \end{aligned}
\end{equation}
with all operations implemented using the appropriate CVXPY methods, and $\eta$ is a small strictly positive constant used to ensure that $0 < \sigma_{\mathrm{proj},i} < \sigma_{\mathrm{base},i}$. Implementation is mostly straightforward except for the term $\frac{1}{2} \sum_{i=1}^{n} \frac{\left(\mu_{\mathrm{proj},i} - \mu_{\mathrm{base},i}\right)^{2}}{\sigma_{\mathrm{base},i}^{2} - \sigma_{\mathrm{proj},i}^{2}}$, since optimization variables are in both the numerator and the denominator, which is generally not DCP. To overcome this, we must use the CVXPY method $\mathrm{quad\_over\_lin}(\mathbf{X},y) = \frac{\sum_{i,j}X_{ij}^{2}}{y}$ and sum over the action space dimensions:

\begin{equation}
    \frac{1}{2} \sum_{i=1}^{n} \frac{\left(\mu_{\mathrm{proj},i} - \mu_{\mathrm{base},i}\right)^{2}}{\sigma_{\mathrm{base},i}^{2} - \sigma_{\mathrm{proj},i}^{2}}
    = \frac{1}{2} \sum_{i=1}^{n} \mathrm{quad\_over\_lin}(\mu_{\mathrm{proj},i} - \mu_{\mathrm{base},i},\sigma_{\mathrm{base},i}^{2} - \sigma_{\mathrm{proj},i}^{2}).
\end{equation}

Since we will be running the same minimization problem many times but with different base and task policy means and standard deviations, we can achieve a significant speed-up in performance by making the problem adhere to Disciplined Parameterized Programming (DPP) rules, and treating the base and task policy means and standard deviations as parameters. It is straightforward to make the problem DPP with little further work. Note also that in the case of optimization failure, we can always fall back on $\policybase$.

The implementation of the convex optimization problem was tested in isolation using randomly-generated examples for  $\policybase$ and $\policytask$. The $\policybase$ means were sampled using a zero-mean Gaussian distribution, and the $\policybase$ standard deviations were sampled using a Rayleigh distribution. The $\policytask$ means were sampled using a much narrower Gaussian distribution centered on the $\policybase$ means, and the $\policytask$ standard deviations were sampled by subtracting sampled values of a much narrower (and truncated) Rayleigh distribution from the $\policybase$ standard deviations, since in practice, the trained $\policytask$ would be somewhat close to $\policybase$ but with lower variance.

Figure \ref{fig:proj_example} presents the solution to the optimization problem for one such example, with $\alpha=1.1$. We see that the solution for $\policyproj$ is such that KL divergence from $\policytask$ is minimized while remaining within $\mathit{\Pi}_{\alpha,\policybase}$. Both the DCP and DPP problems typically return the same solution to 4 decimal places.

The relationship between mean solve time and number of action space dimensions was also investigated. Figure \ref{fig:solve_times} presents the results for both the DCP and DPP problems. We note a linear increasing trend in both cases, and that the DPP problem is roughly an order of magnitude faster to solve than the DCP problem. The solve time remains in the order of milliseconds for the DPP problem for even high-dimensional action spaces, so it is feasible to apply the optimization scheme to even high-dimensional, high-frequency control problems.

\begin{figure}[t]
\centering
\subfloat[]{\includegraphics[width=0.49\linewidth]{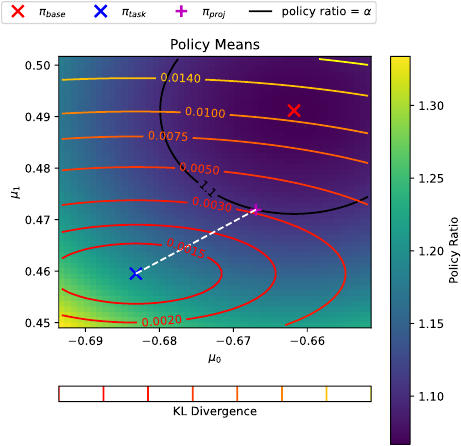}%
\label{fig:proj_example:proj_means}}
\hfil
\subfloat[]{\includegraphics[width=0.49\linewidth]{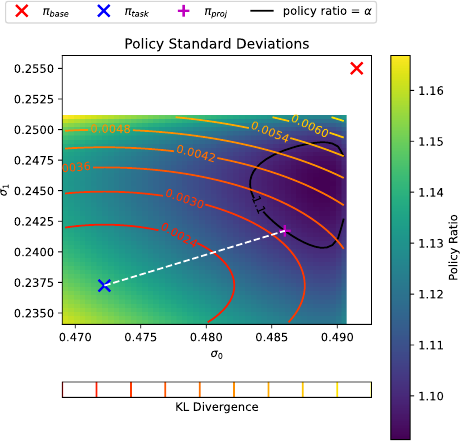}%
\label{fig:proj_example:proj_stdevs}}
\caption{Example result from the convex minimization problem for a 2D action space with randomly generated base policy $\policybase$ and $\policytask$, with $\alpha=1.1$; the black contour depicts the $\alpha = 1.1$ level set. The white dotted line represents the projection of $\policytask$ onto the set $\mathit{\Pi}_{\alpha,\policybase}$, producing $\policyproj$. (\ref{fig:proj_example:proj_means}) Variation in policy ratio and KL divergence with policy means (standard deviations fixed at the values for the joint optimum). (\ref{fig:proj_example:proj_stdevs}) Variation in policy ratio and KL divergence with policy standard deviations (means fixed at the values for the joint optimum). Note how the solution for $\policyproj$ is such that KL divergence from $\policytask$ is minimized while remaining within $\mathit{\Pi}_{\alpha,\policybase}$.}
\label{fig:proj_example}
\end{figure}

\begin{figure}[t]
    \centering
    \includegraphics[width=0.65\linewidth]{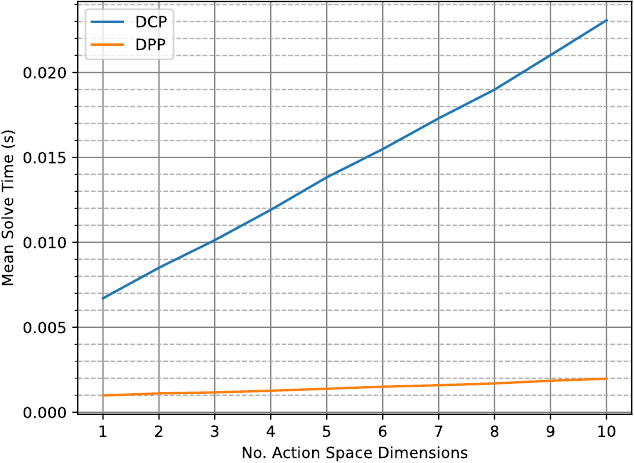}
    \caption{Mean time to solve the optimization program (across 1000 examples) over the number of action space dimensions. We see a linear increasing relationship between the two. Note that the DPP problem is roughly an order of magnitude faster to solve than the DCP problem.}
    \label{fig:solve_times}
\end{figure}

\section{Case Studies}

\subsection{LTL Formula and Robustness Metric for Time-Bounded Reach-Avoid}
\label{appendix:LTL_case_studies}

The LTL formula for the time-bounded reach-avoid property is given as follows:
\begin{equation}
    \varphi = ((\lnot h) \mathbf{U} g) \wedge (\mathbf{F}_{\leq T} g),
\end{equation}
where $g$ and $h$ are atomic propositions for entering the goal and hazard respectively, $T$ is the time limit, and $\lnot$, $\mathbf{U}$, $\wedge$ and $\mathbf{F}_{\leq T}$ are the `not', `until', `and' and `eventually (before $T$)' operators respectively.

A valid robustness metric for this property is given by
\begin{equation}
    \rho^{\varphi}(\traj{}{}) = \sup_{t \in [0..T]} \left( \min \left\{ \inf_{t' \in [0..t]} \Bigl( -d_{\mathcal{S}_H}(\traj{}{}[t']) \Bigr), d_{\mathcal{S}_G}(\traj{}{}[t]) \right\} \right),
\end{equation}
where $\traj{}{}[t]$ is the MDP state $s$ along the trajectory at time $t$, $d_{\Omega}(s)$ is the signed Euclidean distance from MDP state $s$ to some set $\Omega \subset \mathcal{S}$ (where the sign is negative if $s \notin \Omega$), and $\mathcal{S}_G,\mathcal{S}_H \subset \mathcal{S}$ are the goal and hazard sets, respectively. Note that we never have to actually evaluate $\rho^{\varphi}(\traj{}{})$ beyond checking its sign, which is equivalent to checking LTL formula satisfaction, but we provide its explicit representation here to demonstrate its existence (which is required for Theorem \ref{thm:scenario} to hold).

Note that if we choose signal mapping $(x_{G}'(s),x_{H}'(s)) = (d_{\mathcal{S}_G}(s),d_{\mathcal{S}_H}(s))$, this robustness metric is the same as the STL robustness signal for the STL formula for the time-bounded reach-avoid property,
\begin{equation}
    \varphi_{\mathrm{STL}} = \Bigl(\bigl(\lnot (x_{H} \geq 0)\bigr) \mathbf{U} (x_{G} \geq 0)\Bigr) \wedge \bigl(\mathbf{F}_{\leq T} (x_{G} \geq 0)\bigr),
\end{equation}
evaluated on a finite-length, discrete-time trajectory $\traj{}{}$. We see that the LTL atomic propositions $g(s) = \mathcal{H}(d_{\mathcal{S}_G}(s))$ and $h(s) = \mathcal{H}(d_{\mathcal{S}_H}(s))$, in-keeping with our observation in Appendix \ref{appendix:ensuring_existence} that choosing $x'(s)$ such that $L(s) = \mathcal{H}(x'(s))$ allows us to build an equivalent STL formula $\varphi_{\mathrm{STL}}$ for any LTL formula $\varphi$ such that $\traj{}{} \models \varphi \Leftrightarrow \traj{}{} \models \varphi_{\mathrm{STL}}$.

\subsection{MDP Observation and Action Spaces}

The environment for our experimental setup uses continuous observations and actions. There are a total of 44 observations; 32 of these are for the agent's LiDAR sensor, which measures minimum distance from both the goal and hazard sets in each of 16 equally-spaced directions pointing radially out from the agent. This is a reasonable abstraction of a LiDAR typically found on a mobile robot (though a real LiDAR would likely not be able to distinguish between hazard and goal). Nine of the observations are for the agent's accelerometer, velocimeter and gyro (three each), measuring acceleration as well as linear and angular velocity respectively in each spatial dimension; again these are common sensors to have on a real mobile robot (or could at least be estimated using odometry information). Note that one each of the accelerometer and velocimeter observations and two of the gyro observations are irrelevant since the agent is restricted to a plane and so cannot fly up/down, pitch or roll. The final three observations are for the agent's magnetometer that measures magnetic flux in each spatial dimension, which is irrelevant for this problem since there is no magnetic flux.

The two agent actions are forward drive force and turning velocity; this is a reasonable abstraction of a high-level controller for a skid-steering mobile robot.

See \cite{Ji2023} for more detailed information about the Safety Gymnasium agent observations and actions.

\subsection{Training the Base and Task Policy}

SPoRt assumes the availability of a pre-trained $\policybase$, as it focuses on maintaining safety while refining or adapting an existing policy. For our case studies, we obtained $\policybase$ using conventional SAC, which encountered several violations during training. However, these violations, which occur with intermediate policies, do not impact our results, as we only consider the final $\policybase$ and its potential violations. In practice, there are safe training methods for the $\policybase$. For example, in robotics, the $\policybase$ could be trained in a safety-noncritical simulator or using a safety harness, followed by validation in the real environment.

Note that the quality of the trained $\policybase$ will impact its violation probability $\epsilon_{\mathrm{base}}$; higher $\epsilon_{\mathrm{base}}$ means that $\alpha$ must be lower to achieve the same $\epsilon_{\mathrm{task}} = \epsilon_{\mathrm{base}} \alpha^{T} \leq \epsilon_{\mathrm{max}}$, reducing the extent to which the policy can be safely fine-tuned. Moreover, a less exploratory $\policybase$ with lower variance will allow less deviation of $\policyproj$ during fine-tuning for a given $\alpha$. Therefore, $\policybase$ was trained so as to achieve a high probability of property satisfaction (low $\epsilon_{\mathrm{base}}$) while remaining fairly exploratory. This was achieved by training $\policybase$ using SAC with the following sparse reward scheme:
\begin{equation}
    r_{\policybase}(s) =
    \begin{cases}
        \begin{aligned}
            & +1 & & s \in \mathcal{S}_{G} \\
            & -0.25 & & s \in \mathcal{S}_{H} \\
            & ~ 0 & & \mathrm{otherwise}.
        \end{aligned}
    \end{cases}
\end{equation}
The exploratory behavior was encouraged by setting a target (differential) entropy of $0$. To speed up training we used a curriculum \cite{Bengio2009} with 5 stages of increasing difficulty (the final stage being that used for the task environment in the experiments) as well as Prioritized Experience Replay (PER) \cite{Schaul2016}.

In both cases, training was done in the task environment used for experiments (no curriculum). For Case 1, $\policytask$ was trained using SAC but using (dense) $r_{\mathrm{task}}$ without PER; to speed up training, since we assume separately training $\policytask$ without safety considerations is acceptable for Case 1, the environment was not reset when the agent entered the hazard. For Case 2, a different $\policytask$ was trained using Projected PPO for each value of $\alpha$. Training hyperparameters can be found in Table \ref{tab:training_sac} and \ref{tab:training_projected_ppo}. Figure \ref{fig:projected_ppo_episodic_return_alpha_100} presents episodic return during training for Projected PPO with $\alpha = 100$; we see a gradual rise over time as $\policytask$ begins to deviate from $\policybase$, eventually reaching a plateau as $\policytask$ nears the boundary of $\Pi_{\alpha,\policybase}$.

\begin{table}
    \centering
    \begin{tabular}{lrr}
        \toprule
        & $\policybase$ & $\policytask$ (Case 1) \\
        \midrule
        Total interactions & \num{1e5} & \num{1e5} \\
        Policy learning rate & \num{2e-4} & \num{3e-4} \\
        Critic learning rate & \num{7e-4} & \num{1e-3} \\
        Adam epsilon & \num{1e-8} & \textemdash \\
        Discount factor & 0.99 & \textemdash \\
        Buffer size & \num{1e5} & \textemdash \\
        Batch size & 512 & 256 \\
        Replay buffer burn-in before training & \num{5e3} & \textemdash \\
        Autotune entropy coefficient & True & \textemdash \\
        Target entropy per action & 0 & -0.41 \\
        Policy update frequency & 2 & \textemdash \\
        Critic update frequency & 1 & \textemdash \\
        Tau & \num{5e-3} & \textemdash \\
        Use PER & True & False \\
        PER alpha & 0.6 & N/A \\
        PER beta start & 0.4 & N/A \\
        PER beta end & 1.0 & N/A \\
        Use curriculum & True & False \\
        Curriculum levels & 5 & N/A \\
        Curriculum success rate threshold & 0.95 & N/A \\
        Curriculum success window & 100 & N/A \\
        Hidden layers & [256,256] & \textemdash \\
        Activation & ReLU & \textemdash \\
        \bottomrule
    \end{tabular}
    \caption{SAC hyperparameters. Dashes (\textemdash) indicate shared hyperparameter.}
    \label{tab:training_sac}
\end{table}

\begin{table}
    \centering
    \begin{tabular}{lrr}
        \toprule
        & $\policytask$ (Case 2) \\
        \midrule
        Total interactions & \num{3e4} \\
        Learning rate & \num{5e-5} \\
        Adam epsilon & \num{1e-8} \\
        Discount factor & 0.99 \\
        Anneal learning rate & False \\
        Steps per batch & 128 \\
        Minibatches & 4 \\
        Update epochs & 4 \\
        Normalize advantage & True \\
        GAE & True \\
        GAE lambda & 0.95 \\
        Clip coefficient & 0.2 \\
        Clip value loss & False \\
        Entropy coefficient & 0 \\
        Value loss coefficient & 0.5 \\
        Max gradient norm & 0.5 \\
        State-dependent STD & True \\
        Hidden layers & [256,256] \\
        Activation & ReLU \\
        Warm-start interactions & \num{2.5e3} \\
        Warm-start learning rate & \num{3e-4} \\
        \bottomrule
    \end{tabular}
    \caption{Projected PPO hyperparameters.}
    \label{tab:training_projected_ppo}
\end{table}

\begin{figure}[t]
    \centering
    \includegraphics[width=0.99\linewidth]{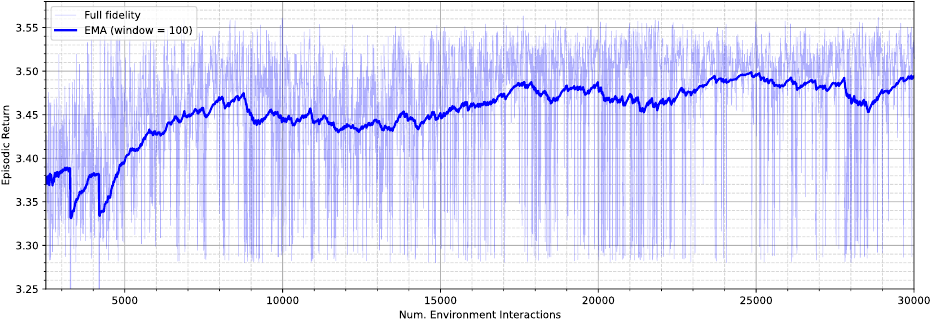}
    \caption{Episodic return during training for Projected PPO with $\alpha = 100$. Both the full fidelity data and the exponential moving average (window size 100) are plotted.}
    \label{fig:projected_ppo_episodic_return_alpha_100}
\end{figure}

\subsection{Selecting Optimal Maximum Episode Length}

Our prior bound $\eps_{\mathrm{task}} = \eps_{\mathrm{base}} \alpha^{T}$ grows exponentially with maximum episode length $T$. For many tasks, the agent is likely to achieve property satisfaction far before reaching the time limit, so it is useful to reduce $T$ to keep $\eps_{\mathrm{task}}$ as low as possible. Note that reducing $T$ can affect the number of scenarios that violate property $\varphi$, which changes the value of $\eps_{\mathrm{base}}$, and so we must now treat $\eps_{\mathrm{base}} = \eps_{\mathrm{base}}(T)$ as a function of $T$. To compute $\eps_{\mathrm{base}}(T)$ with confidence $1 - \beta$ we can use Corollary 1 where $k = k(T)$ is the number of scenarios that violate $\varphi$ when the scenario trajectory length is reduced to $T$.

To maximize improvement in task-specific performance we would like to choose $T$ so as to maximize $\alpha$ at the user-specified maximum acceptable bound on property violation $\eps_{\mathrm{task}}$, since a larger $\alpha$ allows for greater deviation in $\policyproj$ from $\policybase$, enabling greater task-specific performance. Of course, this search can be done automatically, and can be included in the SPoRt pipeline just prior to starting Projected PPO.

Figure \ref{fig:selecting_optimal_episode_length:failure_probs_over_episode_length} presents, for our time-bounded reach-avoid experiment, $\eps_{\mathrm{task}} = \eps_{\mathrm{base}}(T) \alpha^{T}$ from $T = 0$ to $T = 100$ (the original maximum episode length during training) for different values of $\alpha$. We see that for all $\alpha > 1$ there exists $0 < T < 100$ at which $\eps_{\mathrm{task}}$ is minimized.

Figure \ref{fig:selecting_optimal_episode_length:get_optimal_episode_length} presents the maximum permitted $\ln(\alpha)$ over $0 < T \leq 100$ for different values of user-specified $\eps_{\mathrm{task}}$. To select the optimal $T$, we simply choose $T$ at which $\ln(\alpha)$ is maximized for the chosen $\eps_{\mathrm{task}}$. For our experiments we chose $\eps_{\mathrm{task}} = 1$ and so we obtained $T = 21$ and $\ln(\alpha) = 0.22$ ($\alpha = 1.246$).

\begin{figure}[t]
    \centering
    \subfloat[]{\includegraphics[width=0.49\linewidth]{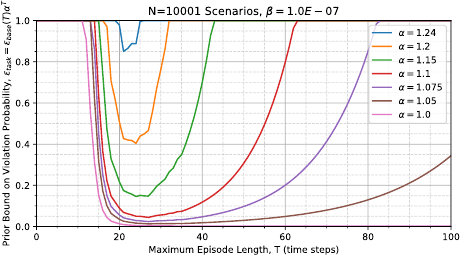}
    \label{fig:selecting_optimal_episode_length:failure_probs_over_episode_length}}
    \hfil
    \subfloat[]{\includegraphics[width=0.49\linewidth]{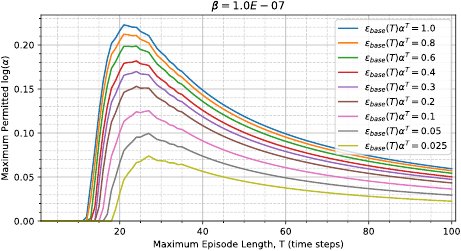}
    \label{fig:selecting_optimal_episode_length:get_optimal_episode_length}}
    \caption{Selecting optimal maximum episode length $T$. (\ref{fig:selecting_optimal_episode_length:failure_probs_over_episode_length}) Violation probabilities over episode length. (\ref{fig:selecting_optimal_episode_length:get_optimal_episode_length}) Maximum permitted $\alpha$ over episode length.}
    \label{fig:selecting_optimal_episode_length}
\end{figure}

\subsection{Additional Figures and Discussion}

Figure \ref{fig:episode:pretrained_episode} presents a more detailed view of the agent behavior over an example episode for Case 1, for $\alpha = 5$ (representing a compromise between safety and performance). Looking at mean turning velocity over the episode, we see that while both $\policybase$ and $\policytask$ drive the agent clockwise around the hazard, $\policytask$ induces sharper turning (much like for Case 2) until around $t = 7$, at which point $\policytask$ drives the agent in roughly the same direction as $\policybase$ but with higher forward drive force. Again, we see that $\policyproj$ always lies within the $\alpha = 5$ level set.

\begin{figure}[h]
    \centering
    \includegraphics[width=0.99\linewidth]{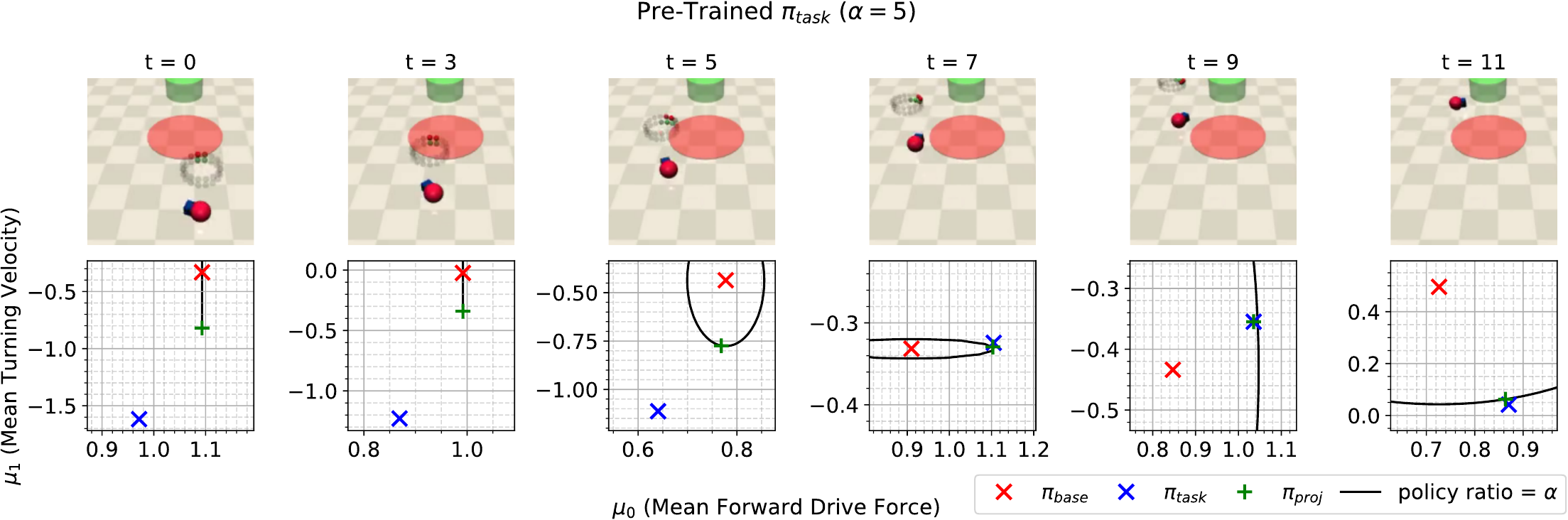}
    \caption{Snapshots across an example episode of the reach-avoid experiment using $\policyproj$ for Case 1 (pre-trained $\policytask$) and $\alpha = 5$; the bottom plots present the action means at the corresponding time step, with the black contour depicting the $\alpha = 5$ level set. Note that positive mean turning velocity represents anticlockwise rotation. The halo above the agent is a visualization of its LiDAR observations for the hazard and goal.}
    \label{fig:episode:pretrained_episode}
\end{figure}

Figures \ref{fig:plots_closeup:failure_probs_closeup} and \ref{fig:plots_closeup:mean_std_time_taken_closeup} plot violation probabilities and mean (and standard deviation) episode length for successful trajectories respectively, and are the same as those found in the main paper but zoomed in to the scale across which $\eps_{\mathrm{task}} \leq 1$, for the reader's convenience.

\begin{figure}[h]
    \centering
    \subfloat[]{\includegraphics[width=0.471\linewidth]{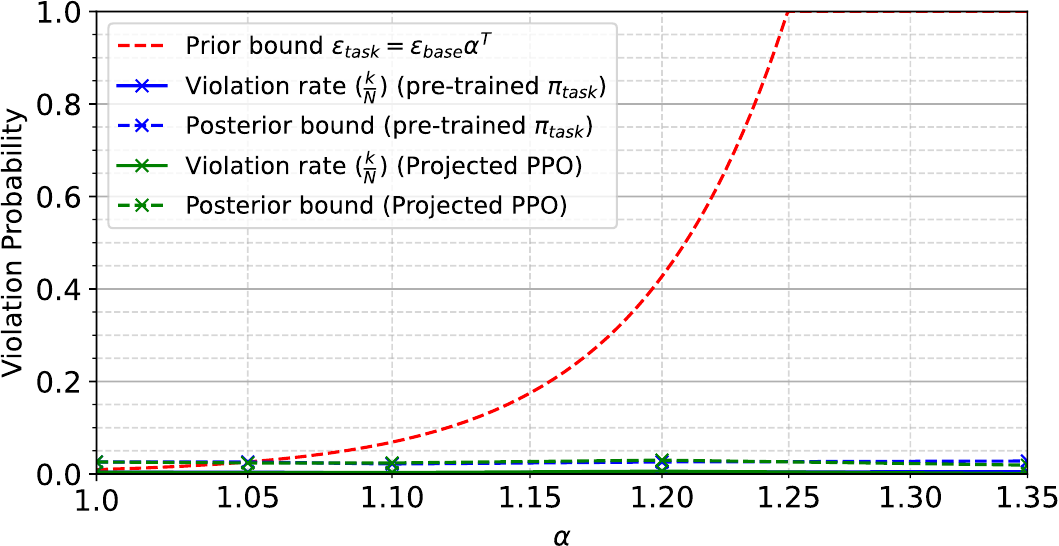}
    \label{fig:plots_closeup:failure_probs_closeup}}
    \hfil
    \subfloat[]{\includegraphics[width=0.514\linewidth]{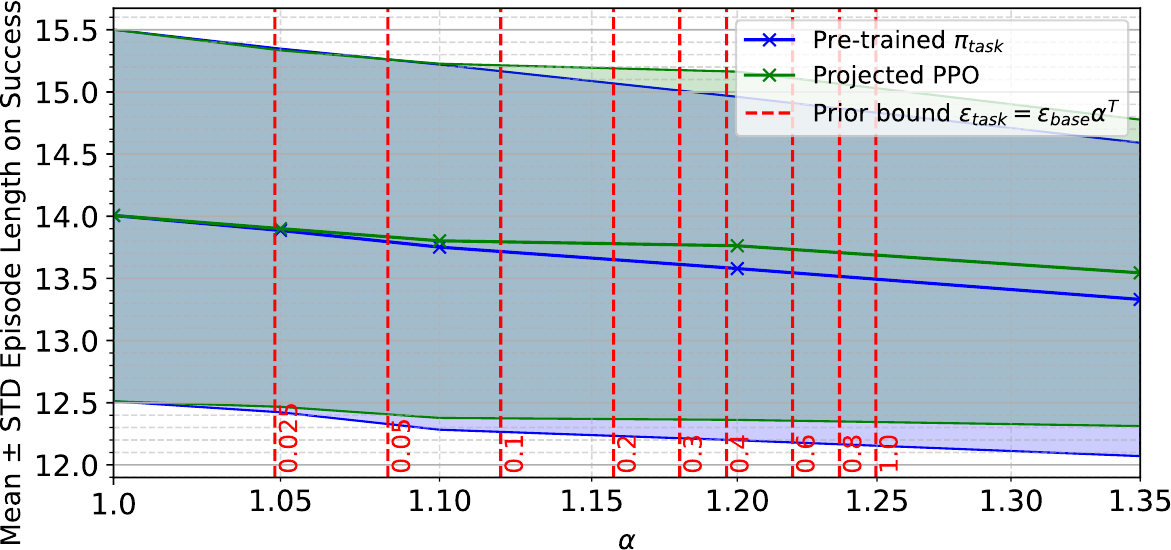}
    \label{fig:plots_closeup:mean_std_time_taken_closeup}}
    \caption{Results from the reach-avoid experiment for both Case 1 (pre-trained $\policytask$) and 2 ($\policytask$ trained using Projected PPO), zoomed in to the scale across which $\eps_{\mathrm{task}} \leq 1$. (\ref{fig:plots_closeup:failure_probs_closeup}) Violation probabilities over different values of $\alpha$. (\ref{fig:plots_closeup:mean_std_time_taken_closeup}) Mean and standard deviation episode length for successful trajectories for different values of $\alpha$. Action seeding for each episode was controlled across different values of $\alpha$ and across the different cases, so all results depend on $\alpha$ and the training of $\policytask$.}
    \label{fig:plots_closeup}
\end{figure}

\newpage

\bibliographystyle{unsrt}
\bibliography{references}